\documentclass[11pt]{article}
\usepackage[preprint]{acl}
\usepackage{times,latexsym}
\usepackage[T1]{fontenc}
\usepackage[utf8]{inputenc}
\usepackage{microtype,inconsolata}
\usepackage{amsmath,amssymb,amsthm,booktabs,tabularx,longtable,graphicx}
\usepackage{algorithm,algpseudocode}
\usepackage{placeins,etoc}
\usepackage{enumitem}
\hypersetup{pdftitle={Beyond Expert Rankings: Contrast Calibration for Mixed-Precision MoE Quantization}}
\providecommand{\doi}[1]{doi: \href{https://doi.org/#1}{{\urlstyle{rm}\nolinkurl{#1}}}}
\newtheorem{theorem}{Theorem}[section]
\newtheorem{proposition}[theorem]{Proposition}

\newtheorem{corollary}[theorem]{Corollary}

\newcommand{\cI}{\mathcal I}

\newcommand{\method}{NatScale}
\DeclareRobustCommand{\appref}[1]{Appendix~\ref{#1}}
\DeclareRobustCommand{\secref}[1]{Section~\ref{#1}}
\DeclareRobustCommand{\equref}[1]{Equation~\eqref{#1}}

\DeclareRobustCommand{\tabref}[1]{Table~\ref{#1}}

\DeclareMathOperator*{\argmin}{arg\,min}
\title{Beyond Expert Rankings: Contrast Calibration for Mixed-Precision MoE Quantization}
\newcommand{\NatScaleAuthorBlock}{%
\begin{minipage}{0.99\textwidth}
\centering
{\normalsize\bfseries
Renjie Xie\textsuperscript{1,2}, Juncheng Yang\textsuperscript{2},
Aoting Hu\textsuperscript{3,2}, Yan Liang\textsuperscript{5}\\
Mingxi Zhang\textsuperscript{1,2}, Zheheng Hong\textsuperscript{4,2},
Wei Xu\textsuperscript{5}, Yanzhi Chen\textsuperscript{6,7}\par}
\vspace{0.35em}
{\normalfont\small
\textsuperscript{1}Nanjing University of Posts and Telecommunications\quad
\textsuperscript{2}Tylogi AI Lab / TAIL\\
\textsuperscript{3}Anhui University of Technology\quad
\textsuperscript{4}Shanghai Jiao Tong University\\
\textsuperscript{5}Southeast University\quad
\textsuperscript{6}Microsoft\quad
\textsuperscript{7}University of Cambridge\par}
\vspace{0.25em}
{\normalfont\footnotesize\ttfamily
renjie\_xie@njupt.edu.cn, juncheng\_yang@tylogi.com, aotinghu@ahut.edu.cn\\
liang\_yan@seu.edu.cn, b23080111@njupt.edu.cn, zh2227@sjtu.edu.cn\\
wxu@seu.edu.cn, yc514@cam.ac.uk\par}
\end{minipage}%
}
\author{\NatScaleAuthorBlock}
\hypersetup{pdfauthor={Renjie Xie; Juncheng Yang; Aoting Hu; Yan Liang; Mingxi Zhang; Zheheng Hong; Wei Xu; Yanzhi Chen},pdfsubject={Mixed-precision MoE quantization}}

\begin{document}
\etocdepthtag.toc{main}
\maketitle
\begin{abstract}
Mixture-of-experts (MoE) language models activate only a few experts per token, yet all expert weights must be stored. Mixed-precision quantization reduces this storage cost by assigning higher bit-widths to sensitive experts and lower bit-widths to resilient ones, but expert rankings alone do not determine how strongly to weight their errors. To address this gap, we propose NatScale, a contrast-calibrated allocation framework for mixed-precision MoE quantization. It uses a model-specific power transformation to calibrate gradient-based expert weights to an information-theoretic contrast target while preserving rankings. We jointly evaluate each expert's gate, up and down configurations to estimate its output error. We then combine these errors and calibrated weights with an existing prior based on weight spectra, variance and precision decay to select configurations under a storage budget. Across three models, NatScale achieves the lowest native-to-quantized KL divergence and the highest top-token agreement in all nine settings against five allocation baselines. It also yields the lowest perplexity in seven settings. These findings highlight the importance of objective-aware precision allocation.
\end{abstract}

\section{Introduction}
\label{sec:introduction}

Sparse Mixture-of-Experts (MoE) language models decouple parameter count from computational cost by routing each token to a small subset of experts~\citep{jiang2024mixtral,deepseek2024v2}. Although this routing limits per-token computation, all expert parameters must still be stored. Mixed-precision quantization offers an effective way to compress these models: under a strict memory budget, it allocates higher precision to sensitive weights and lower precision to robust ones.

Prior work uses gradients or curvature information to estimate
how quantization errors affect the model's loss,
providing a basis for sensitivity-aware quantization
\citep{kim2024squeezellm,kim2025guidedquant}. Choosing precisions under this budget requires two ingredients: estimating the quantization error of candidate configurations and evaluating the model's sensitivity to those errors. Existing techniques rank experts by importance using diverse heuristics, including routing statistics and reconstruction errors~\citep{huang2025mcmoe}, block sensitivity combined with activations~\citep{duanmu2025mxmoe}, dynamic importance updates during quantization~\citep{deng2026gemq}, and weight spectra priors~\citep{yang2026alphaq}.

While these signals establish expert rankings, a ranking alone does not determine memory allocation: it indicates which experts are more important, but not how strongly to weight their errors. Consider two sensitivity score vectors: $(2, 1)$ and $(100, 1)$. Both vectors rank the first expert ahead of the second, but assign different relative weights to their errors. We refer to the magnitude of this relative separation between sensitivity scores as \emph{contrast}. Under a global storage budget, contrast controls how strongly experts compete for storage. AlphaQ adapts the curvature of its spectral importance mapping to model-weight statistics~\citep{yang2026alphaq}. We study this weighting strength for gradient-based expert sensitivities.

In this work, we propose \textbf{NatScale}, an allocation framework designed to explicitly control sensitivity contrast in mixed-precision MoE quantization. NatScale applies a power transformation to gradient-based expert sensitivities, choosing the exponent so that normalized squared weights reach a prescribed Jeffreys (symmetric KL) divergence from a uniform distribution over positive-score experts. Because a power transformation is order-preserving, this calibration leaves the expert rankings untouched, and the prescribed Jeffreys divergence expresses each model's overall contrast on a common information-theoretic scale.

To score candidate configurations with these calibrated sensitivity weights, NatScale incorporates two additional design elements: (1) \emph{Joint Projection Error Measurement}. Because gate and up outputs interact before the down projection, separately measured errors may miss their joint effect on the output. We measure output error with all three projections quantized jointly on fixed original-model inputs. (2) \emph{Spectral Structural Priors}. The prior itself summarizes each block's weight structure through three factors: weight spectra, variance and precision decay. Multiplying this prior by the joint errors and calibrated sensitivities yields a score for each configuration, and these scores determine the selection under a global storage budget.

The contributions of our work are:
\begin{itemize}[leftmargin=*]
\item We propose \textbf{NatScale} for mixed-precision MoE quantization. Rather than using importance scores directly or applying fixed power transformations, NatScale calibrates expert sensitivity contrast to an explicit information-theoretic target, allowing for more effective allocation of quantization precision under a fixed  budget.
\item We show that expert rankings alone do not determine precision allocation: sensitivity contrast controls how strongly experts compete for storage. We further demonstrate that preserving native outputs and minimizing perplexity can favor different contrast levels.
\item We evaluate NatScale on WikiText-2 across Qwen3.6-35B-A3B, Qwen3.8-Flash-Next, and DeepSeek-V4-Flash-0731 at three storage targets, achieving the best native-output fidelity among the tested methods in all nine settings. We additionally evaluate FlashNext on AIME, GPQA, TruthfulQA, and LiveCodeBench.
\end{itemize}

\section{Related Work}
\label{sec:related-work}

\noindent\textbf{Post-Training Quantization.}
Post-training quantization (PTQ) reduces quantization error: GPTQ~\citep{frantar2023gptq} and AWQ~\citep{lin2024awq} use calibration activations, while QuIP\#~\citep{tseng2024quipsharp} and AQLM~\citep{egiazarian2024aqlm} employ codebooks for low-bit quality. Sensitivity-aware quantization methods like BRECQ~\citep{li2021brecq}, SqueezeLLM~\citep{kim2024squeezellm}, GuidedQuant~\citep{kim2025guidedquant}, and KronQ~\citep{lee2026kronq} compute importance via gradients or curvature, with DASH-Q~\citep{kim2026dashq} mitigating noisy curvature. For precision allocation, SliM-LLM~\citep{huang2025slimllm} assigns bits to groups according to their salience, whereas CoopQ~\citep{zhao2026coopq} uses layer interactions to allocate precision across layers. These dense-model formulations do not directly specify expert-level choices under sparse routing.

\noindent\textbf{MoE Quantization.}
Extending precision allocation to Mixture-of-Experts (MoE) models requires capturing expert importance and interactions. MC-MoE~\citep{huang2025mcmoe} scores experts via activation frequency, routing probabilities, and error. MODE~\citep{chen2026mode} estimates modality-specific importance, and GEMQ~\citep{deng2026gemq} updates sensitivity dynamically. MxMoE~\citep{duanmu2025mxmoe} and FRI-MxMoE~\citep{qi2026frimxmoe} allocate block-level precision using reconstruction costs or fuzzy interpolation, while AlphaQ~\citep{yang2026alphaq} uses weight spectra and adapts the curvature of its importance mapping to model statistics. Recent work refines granularity: SFMoE~\citep{qi2027sfmoe} models cooperation, CAMERA~\citep{xu2026camera} constructs micro-experts, and Q-Strata~\citep{lee2026qstrata} uses model-level objectives. NatScale calibrates gradient-based expert weights to a prescribed Jeffreys divergence. It jointly evaluates expert configurations, combining output errors with calibrated weights and AlphaQ's structural factor, trading coordinate-level gradient alignment for global tractability.

\noindent\textbf{Calibration and Robustness.}
Compression performance also depends on calibration data and robustness. Dataset selection critically impacts quality~\citep{williams2024calibration}, including self-generated text~\citep{williams2025selfcalibration}, sequence length~\citep{oh2025fixedlength}, and language diversity~\citep{chimoto2026multilingual}. In multimodal models, MBQ~\citep{li2025mbq} uses output gradients to balance reconstruction across vision and language tokens. For MoE models, MoEQuant~\citep{chen2025moequant} improves calibration coverage through expert-balanced sampling and uses sample--expert affinities during quantization. Related work also addresses differences across data groups: Group DRO~\citep{sagawa2020groupdro} optimizes worst-group loss, while FairQuant~\citep{woergaard2026fairquant} uses group-conditioned importance and fairness regularization for mixed-precision quantization. NatScale averages frequency-weighted RMS errors across folds. While its contrast target regulates weight separation, it remains dependent on calibration coverage and distribution shifts.

\begin{figure*}[!t]
\centering
\includegraphics[width=\linewidth]{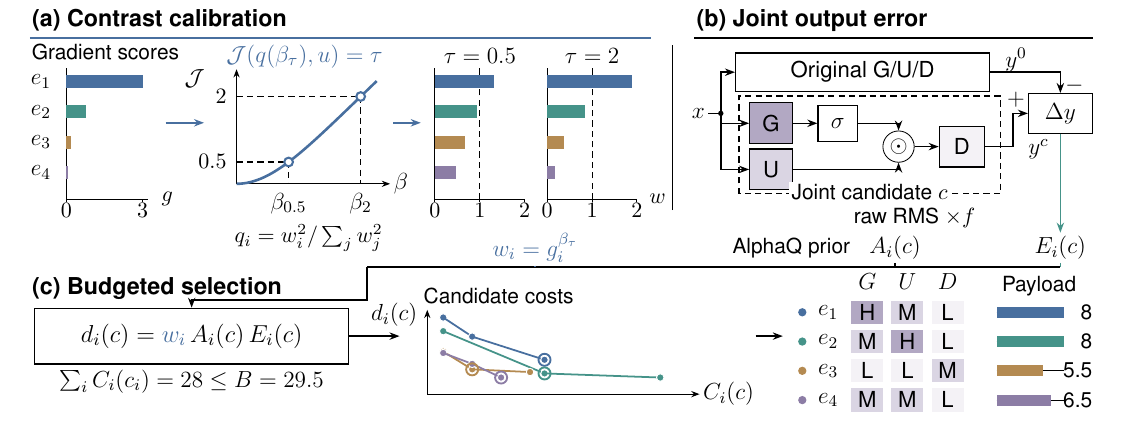}
\caption{\textbf{Overview of NatScale}. \textbf{(a) Contrast Calibration:} Maps gradient scores $g_i$ via a power transformation $w_i = g_i^{\beta_\tau}$, where the exponent $\beta_\tau$ is calibrated using a symmetric KL divergence target $\tau$ to control expert storage competition. \textbf{(b) Joint Error Estimation:} Evaluates quantization distortion by jointly quantizing gate, up, and down projections over replay inputs while incorporating structural priors. \textbf{(c) Budgeted Allocation:} Synthesizes calibrated weights, structural priors, and joint errors into expert costs, solving the global storage allocation by linear-programming (LP) relaxation and feasible rounding under a storage budget $B$.}
\label{fig:natscale-overview}
\end{figure*}

\section{Methodology}
\label{sec:method}

\subsection{Overall Framework}
\label{sec:overview}

We formulate mixed-precision quantization for sparse Mixture-of-Experts (MoE) models as a constrained optimization problem~\citep{jiang2024mixtral, deepseek2024v2}. Let $i = (\ell, e) \in \cI$ index expert $e$ in layer $\ell$, and let $p \in \{G, U, D\}$ denote its gate, up, or down projections. For each expert, a local configuration $c_i = (c_{i,G}, c_{i,U}, c_{i,D}) \in \mathcal C_i$ specifies an encoding format and quantized weights. Given an encoded-weight payload $C_i(c_i)$ in exact bytes and a global storage budget $B$, our goal is to find an allocation $c = (c_i)_{i \in \cI}$ that minimizes an empirical proxy for total model perturbation:
\begin{equation}
  \begin{aligned}
  \min_{\{c_i \in \mathcal C_i\}} \quad & \sum_{i \in \cI} d_i(c_i) \\
  \text{subject to} \quad & \sum_{i \in \cI} C_i(c_i) \le B,
  \end{aligned}
  \label{ns:allocation}
\end{equation}
where $d_i(c_i)$ is the expert's proxy cost and non-routed weights retain their native precision.

As illustrated in Figure~\ref{fig:natscale-overview}, \textbf{NatScale} achieves this allocation through a principled three-stage pipeline. First, it derives gradient-based expert sensitivities and calibrates their relative separation, termed \emph{contrast}, using an information-theoretic symmetric KL target. Second, it evaluates candidate distortions by jointly quantizing all three projections while incorporating a structural weight prior~\citep{yang2026alphaq}. Finally, it synthesizes these signals into composite expert costs and solves the global storage allocation using LP relaxation and feasible rounding.

\subsection{Step 1: Sensitivity Calibration}
\label{ns:contrast}

To determine how strongly sensitive experts compete for storage, we measure the response of prediction loss to expert output changes. Let $v_{i,t} = \partial \mathcal L_{\mathrm{CE}} / \partial y_{i,t}$ denote the gradient of summed next-token cross-entropy with respect to expert output $y_{i,t}$ at token position $t$. Over calibration folds, we compute the normalized gradient second moment $g_i$:
\begin{equation}
  \begin{aligned}
    G_i &= \frac{\sum_s \sum_{t \in \mathcal T^{\mathrm{orig}}_{s,i}} \|v_{i,t}\|_2^2}{H_i \max\{\sum_s n^{\mathrm{orig}}_{s,i},1\}}, \\
    g_i &= \frac{G_i}{|\cI|^{-1}\sum_{j\in\cI}G_j},
  \end{aligned}
  \label{ns:gradient}
\end{equation}
where global mean normalization preserves relative inter-layer sensitivities. Here $H_i$ is the expert output width, $\mathcal T^{\mathrm{orig}}_{s,i}$ contains its routed token events in native calibration fold $s$, and $n^{\mathrm{orig}}_{s,i}=|\mathcal T^{\mathrm{orig}}_{s,i}|$.

Rather than relying on raw scores or heuristic powers, \textbf{NatScale} applies a power transformation $w_i = g_i^{\beta_\tau}$ on active experts $\cI_+ = \{i : g_i > 0\}$ and determines $\beta_\tau$ from a target symmetric KL divergence $\tau>0$ against a uniform distribution. For a candidate exponent $\beta\ge0$, the normalized squared weights and uniform reference are
\begingroup
\begin{equation}
  q_i(\beta)=\frac{g_i^{2\beta}}{\sum_{j\in\cI_+}g_j^{2\beta}},
  \qquad u_i=\frac{1}{|\cI_+|}.
  \label{ns:squared-weights}
\end{equation}
\endgroup
When the positive scores are nonconstant, each target $\tau>0$ determines a unique exponent $\beta_\tau>0$ through
\begin{equation}
  \begin{aligned}
    &\mathcal J(q(\beta_\tau), u) \\ = &D_{\mathrm{KL}}(q(\beta_\tau) \,\|\, u) + D_{\mathrm{KL}}(u \,\|\, q(\beta_\tau)) \\
    = &\tau,
  \end{aligned}
  \label{ns:calibration}
\end{equation}
where target contrast $\tau$ is measured in \emph{nats}. Zero-gradient experts retain zero weight; calibration rejects empty support and constant positive scores. This formulation yields calibrated weights $w_i = g_i^{\beta_\tau}$ that strictly preserve expert rank ordering while setting the prescribed sensitivity contrast. We adopt $\tau = 1$~nat as our maximum-entropy reference default in evaluation.

\subsection{Step 2: Joint Error Estimation}
\label{ns:distortion}

An expert's gate and up projections nonlinearly interact before the down projection through gating operator $\Phi_i$ (e.g., $\operatorname{SiLU}(a) \odot b$). We retain their error interactions by quantizing all three projections jointly on fixed replay inputs $x$:
\begin{equation}
  y_i^c(x) = W_{i,D}^{c_D} \Phi_i\!\left(W_{i,G}^{c_G}x, W_{i,U}^{c_U}x\right).
  \label{ns:expert-output}
\end{equation}
Letting $\Delta y_{i,t}(c) = y_i^c(x_{i,t}) - y_i^0(x_{i,t})$ denote the output residual relative to native outputs, we define the exposure-weighted RMS output error $E_i(c)$ using activation frequencies $f_{s,i}$ across folds:
\begin{equation}
  \begin{aligned}
    M_{s,i}(c) &= \frac{\sum_{t \in \mathcal T^{\mathrm{rep}}_{s,i}} \|\Delta y_{i,t}(c)\|_2^2}{\max\{n^{\mathrm{rep}}_{s,i},1\}H_i}, \\
    E_i(c) &= \sum_s \pi_s f_{s,i}\sqrt{M_{s,i}(c)}.
  \end{aligned}
  \label{ns:local-error}
\end{equation}
Here $\mathcal T^{\mathrm{rep}}_{s,i}$ is the replay event set and $n^{\mathrm{rep}}_{s,i}=|\mathcal T^{\mathrm{rep}}_{s,i}|$. Exposure is $f_{s,i}=n^{\mathrm{orig}}_{s,i}/T_s$, where $T_s$ is the token count in fold $s$; the six folds receive equal weights $\pi_s=1/6$.

To incorporate weight structure, we augment the output error with a structural prior $A_i(c)$ following AlphaQ~\citep{yang2026alphaq}, which incorporates spectral tail indices $\alpha_{i,p}$, parameter variance $\operatorname{Var}(W_{i,p})$, and effective bit precision decay:
\begingroup
\begin{equation}
  \begin{aligned}
    A_i(c)&=\frac13\sum_{p\in\{G,U,D\}}A_{i,p}(c_p),\\
    A_{i,p}(c_p)&=\chi_{i,p}(c_p)\frac{\widetilde\alpha}{\alpha_{i,p}}\\
    &\qquad\cdot\operatorname{Var}(W_{i,p})\,2^{-2b_{i,p}(c_p)}.
  \end{aligned}
  \label{ns:structural-factor}
\end{equation}
\endgroup
Here $\widetilde\alpha$ is the median spectral tail index across all target projections, and $b_{i,p}(c_p)$ uses nominal bits for integer formats and effective bits per weight for vector formats. The indicator $\chi_{i,p}(c_p)$ is zero when that projection retains DeepSeek-V4-Flash-0731's native MXFP4 encoding and one otherwise.

\subsection{Step 3: Budget Allocation}
\label{sec:allocation}

With the calibrated sensitivity weights $w_i$, structural priors $A_i(c)$, and joint output distortions $E_i(c)$ established, we formulate the composite expert proxy cost for configuration $c_i$ as:
\begin{equation}
  \boxed{d_i(c_i) = w_i \, A_i(c_i) \, E_i(c_i)}.
  \label{ns:complete-cost}
\end{equation}
Under Lagrangian relaxation, the calibrated weight $w_i$ rescales the global byte price, making higher-precision configurations more attractive for critical experts. We construct each expert's lower convex hull over cost-payload trade-offs and solve the LP relaxation, then round its solution to a feasible allocation. The gap to the LP optimum bounds the returned allocation's optimization error for this proxy (Appendix~\ref{ns:implementation}).

\paragraph{Implementation Note.} Re-calibrating the contrast target $\tau$ or adjusting the global budget $B$ only requires updating $w_i$ and re-running the solver over pre-cached distortion terms $E_i(c)$, requiring no additional model forward or backward passes. Further technical proofs and implementation details are provided in Appendices~\ref{ns:appendix-calibration}, \ref{ns:appendix-residual}, and~\ref{ns:implementation}.

\section{Theoretical Analysis}
\label{sec:theory}
\label{ns:theory}

To substantiate NatScale's design, this section provides theoretical answers to two fundamental questions: 

\noindent\textbf{1. Why is contrast calibration well-posed, and why choose $\tau = 1$~nat?} Raw sensitivity distributions can differ in their overall contrast. NatScale puts this contrast on a common information scale by mapping scores through a power transformation to a prescribed symmetric KL divergence $\tau$.

For any positive, non-constant gradient scores $g_i$, the symmetric KL divergence $\mathcal J(q(\beta), u)$ is a monotonically increasing and continuous function of exponent $\beta\ge0$ mapping from zero to infinity. Consequently, each target contrast $\tau > 0$ uniquely determines a single positive exponent $\beta_\tau$. Changing $\tau$ smoothly scales the relative separation (contrast) between expert weights without altering their rank ordering.

To set a default without tuning the contrast target, we map a two-parameter Weibull fit of amplitude scores $r_i = \sqrt{g_i}$ to a reference distribution whose squared amplitudes have maximum entropy at unit mean. Combining this map with the ordinary maximum-likelihood equation gives the following finite-sample identity:

\begin{corollary}[Weibull Reference Target]
\label{ns:weibull-corollary}
For any positive, nonconstant amplitude scores $r_i$, the unweighted two-parameter Weibull maximum-likelihood fit satisfies $\beta_1 = \widehat{k} / 4$, yielding a calibrated contrast of exactly $\mathcal J(q(\beta_1), u) = 1$~nat.
\end{corollary}

The derivation in Appendix~\ref{ns:appendix-calibration} applies to any positive nonconstant sample, without assuming a Weibull data distribution. Thus, $\tau = 1$~nat acts as an information-theoretic anchor, while task-specific performance determines whether a tighter or wider contrast is preferred.

\noindent\textbf{2. How do calibrated weights regulate memory allocation?} To understand how calibrated weights govern memory allocation during optimization, let $L_i = A_i E_i$ represent the combined local distortion and structural prior cost for expert $i$. Under Lagrangian relaxation with a global byte price $\lambda \ge 0$, the expert-level selection objective minimizes:
\begin{equation}
  \begin{aligned}
    \min_{c_i \in \mathcal C_i} \quad w_i L_i(c_i) + \lambda C_i(c_i).
  \end{aligned}
\end{equation}
Dividing this objective by the calibrated weight $w_i > 0$ yields an equivalent minimization problem:
\begin{equation}
  \begin{aligned}
    \min_{c_i \in \mathcal C_i} \quad L_i(c_i) + \left(\frac{\lambda}{w_i}\right) C_i(c_i).
  \end{aligned}
  \label{ns:lagrangian-price}
\end{equation}
Equation~\ref{ns:lagrangian-price} gives contrast a direct allocation interpretation: larger $w_i$ lowers the effective byte price $\lambda/w_i$ for expert $i$.

When weights are close, their effective byte prices are similar; wider contrast separates these prices and can concentrate higher-precision storage on more heavily weighted experts. This relation explains how contrast changes storage competition, but it does not determine which target best preserves model outputs.

We therefore use $\tau=1$~nat as a statistical reference and assess its effect empirically. Separately, the gap between the returned feasible allocation and the LP optimum gives an upper bound on its optimization error for the fixed empirical proxy (Appendix~\ref{ns:implementation}).

\section{Experiments}
\label{sec:experiments}

We evaluate \method{} on three MoE language models under three routed-weight storage budgets. Our experiments address four questions: (i) how well does contrast-calibrated allocation preserve model predictions relative to existing allocation strategies; (ii) how does the contrast target affect quality when expert rankings are held fixed; (iii) which score components contribute to performance; and (iv) whether the benefits extend to downstream tasks and whole-model compression.

\begin{table*}[t]
\caption{\textbf{WikiText-2 results at three nominal routed-weight budgets.} Lower perplexity (PPL) and native-to-quantized KL divergence (KLD) are better; higher native top-token agreement (Top1, \%) is better. Bold indicates the best result in each column. All methods use the same candidate quantized weights, but retain method-specific allocation constraints and payload/header accounting (\appref{app:baselines}). NatScale uses the fixed contrast target $\tau=1$ nat.}
\label{tab:baselines-main}
\centering\scriptsize
\begingroup
\setlength{\tabcolsep}{3pt}
\begin{tabular*}{\linewidth}{@{\extracolsep{\fill}}l*{9}{r}@{}}
\toprule
\textbf{Qwen3.6-35B-A3B} & \multicolumn{3}{c}{\textbf{2.0 routed bpw}} & \multicolumn{3}{c}{\textbf{2.5 routed bpw}} & \multicolumn{3}{c}{\textbf{3.0 routed bpw}} \\
\cmidrule(lr){2-4}\cmidrule(lr){5-7}\cmidrule(l){8-10}
Method & PPL $\downarrow$ & KLD $\downarrow$ & Top1 $\uparrow$ & PPL $\downarrow$ & KLD $\downarrow$ & Top1 $\uparrow$ & PPL $\downarrow$ & KLD $\downarrow$ & Top1 $\uparrow$ \\
\midrule
MxMoE~\citep{duanmu2025mxmoe} & 8.0690 & 0.2727 & 79.87 & 7.2930 & 0.1489 & 85.30 & 7.0630 & 0.0870 & 88.93 \\
FRI-MxMoE~\citep{qi2026frimxmoe} & 8.4293 & 0.3118 & 78.14 & 7.9069 & 0.2359 & 80.90 & 7.6321 & 0.1992 & 82.55 \\
AlphaQ~\citep{yang2026alphaq} & 7.7790 & 0.2047 & 81.76 & 7.3119 & 0.1236 & 85.70 & \textbf{6.9932} & 0.0671 & 89.49 \\
MODE~\citep{chen2026mode} & 8.4724 & 0.3291 & 78.30 & 7.6407 & 0.1901 & 83.68 & 7.2792 & 0.1196 & 87.21 \\
GEMQ (alloc.)~\citep{deng2026gemq} & 7.9142 & 0.2189 & 81.89 & 7.3333 & 0.1160 & 87.01 & 7.1582 & 0.0762 & 89.85 \\
\textbf{NatScale (ours)} & \textbf{7.6767} & \textbf{0.1988} & \textbf{82.98} & \textbf{7.1507} & \textbf{0.0984} & \textbf{87.94} & 7.0149 & \textbf{0.0575} & \textbf{90.72} \\
\midrule
\textbf{Qwen3.8-Flash-Next} & \multicolumn{3}{c}{\textbf{2.0 routed bpw}} & \multicolumn{3}{c}{\textbf{2.5 routed bpw}} & \multicolumn{3}{c}{\textbf{3.0 routed bpw}} \\
\cmidrule(lr){2-4}\cmidrule(lr){5-7}\cmidrule(l){8-10}
Method & PPL $\downarrow$ & KLD $\downarrow$ & Top1 $\uparrow$ & PPL $\downarrow$ & KLD $\downarrow$ & Top1 $\uparrow$ & PPL $\downarrow$ & KLD $\downarrow$ & Top1 $\uparrow$ \\
\midrule
MxMoE~\citep{duanmu2025mxmoe} & 6.1777 & 0.5873 & 73.14 & 5.3389 & 0.4109 & 78.66 & 4.9619 & 0.2765 & 82.69 \\
FRI-MxMoE~\citep{qi2026frimxmoe} & 5.5050 & 0.4766 & 75.78 & 5.2422 & 0.4055 & 77.79 & 5.1490 & 0.3607 & 79.04 \\
AlphaQ~\citep{yang2026alphaq} & 5.1147 & 0.3445 & 79.42 & \textbf{4.7987} & 0.2487 & 82.82 & 4.8261 & 0.1381 & 86.72 \\
MODE~\citep{chen2026mode} & 5.5368 & 0.4446 & 76.93 & 4.9840 & 0.2594 & 82.48 & 4.7954 & 0.1663 & 86.04 \\
GEMQ (alloc.)~\citep{deng2026gemq} & 6.0587 & 0.5370 & 73.64 & 5.2986 & 0.3382 & 79.56 & 4.9891 & 0.2243 & 83.50 \\
\textbf{NatScale (ours)} & \textbf{5.0678} & \textbf{0.3293} & \textbf{80.22} & 4.8344 & \textbf{0.1941} & \textbf{84.86} & \textbf{4.7430} & \textbf{0.1243} & \textbf{87.67} \\
\midrule
\textbf{DeepSeek-V4-Flash-0731} & \multicolumn{3}{c}{\textbf{2.0 routed bpw}} & \multicolumn{3}{c}{\textbf{2.5 routed bpw}} & \multicolumn{3}{c}{\textbf{3.0 routed bpw}} \\
\cmidrule(lr){2-4}\cmidrule(lr){5-7}\cmidrule(l){8-10}
Method & PPL $\downarrow$ & KLD $\downarrow$ & Top1 $\uparrow$ & PPL $\downarrow$ & KLD $\downarrow$ & Top1 $\uparrow$ & PPL $\downarrow$ & KLD $\downarrow$ & Top1 $\uparrow$ \\
\midrule
MxMoE~\citep{duanmu2025mxmoe} & 21.1285 & 1.6657 & 58.43 & 16.4595 & 1.3704 & 62.85 & 9.8755 & 0.8292 & 71.11 \\
FRI-MxMoE~\citep{qi2026frimxmoe} & 12.5689 & 1.1318 & 65.65 & 9.7246 & 0.8756 & 69.59 & 7.4664 & 0.6025 & 74.56 \\
AlphaQ~\citep{yang2026alphaq} & 9.0467 & 0.7834 & 71.61 & 6.3970 & 0.4146 & 79.38 & 5.6550 & 0.2259 & 84.88 \\
MODE~\citep{chen2026mode} & 15.2631 & 1.3185 & 63.07 & 9.1905 & 0.7947 & 71.26 & 6.4961 & 0.4273 & 78.63 \\
GEMQ (alloc.)~\citep{deng2026gemq} & 8.9568 & 0.8201 & 70.05 & 6.1106 & 0.3652 & 81.01 & 5.5451 & 0.2103 & 85.98 \\
\textbf{NatScale (ours)} & \textbf{8.3647} & \textbf{0.7293} & \textbf{72.15} & \textbf{5.9615} & \textbf{0.3361} & \textbf{81.57} & \textbf{5.4883} & \textbf{0.1936} & \textbf{86.16} \\
\bottomrule
\end{tabular*}
\endgroup

\end{table*}

\begin{figure*}[t]
\centering
\includegraphics[width=\linewidth]{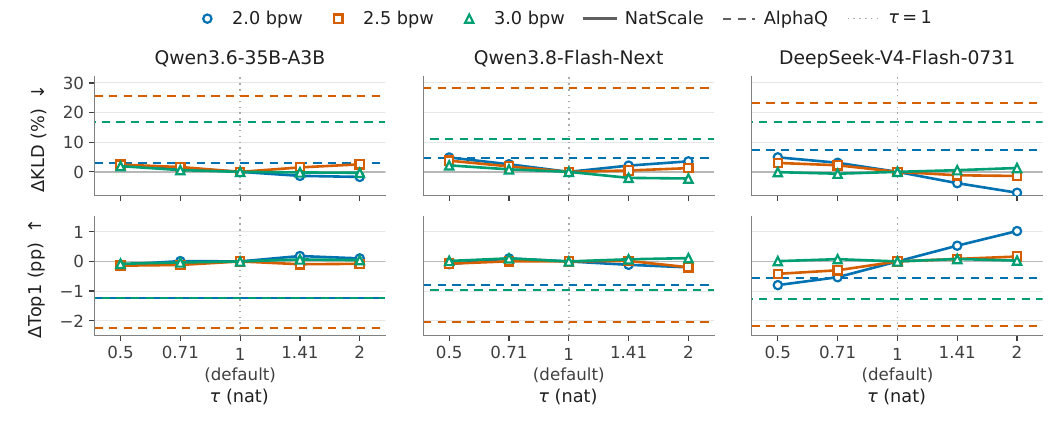}
\caption{\textbf{Sensitivity to the contrast target $\tau$.} Changes in KLD (\%) and Top1 (percentage points) are measured relative to NatScale at $\tau=1$ for each model and nominal routed-bpw budget. Solid curves show NatScale allocations at different targets; same-color horizontal dashed lines show the corresponding AlphaQ reference. The vertical dotted line marks the fixed default.}
\label{ns:tau-curves}
\end{figure*}

\subsection{Experimental Setup}
\label{sec:experimental-setup}

\noindent\textbf{Models.}
We study Qwen3.6-35B-A3B, Qwen3.8-Flash-Next, and DeepSeek-V4-Flash-0731, with reported total parameter counts of 35B, 125B, and 284B, respectively~\citep{qwen2026qwen36,qwen2026flashnext,deepseek2026v4,deepseek2026flash0731}. They activate approximately 3B, 6B, and 13B parameters per token and route each token to 8 of 256, 10 of 512, and 6 of 256 routed experts, respectively, in addition to a shared expert. The models span different sparse architectures and native weight formats. Details are specified in \appref{app:baselines}.

\noindent\textbf{Quantization and calibration.}
We allocate formats to the gate, up, and down projections of routed experts at nominal targets of 2.0, 2.5, and 3.0 bits per weight (bpw). All methods draw from the same precomputed candidate weights produced by MixedFlowQuant, while preserving their respective allocation rules. NatScale constrains encoded payload bytes and accounts for metadata separately; the adapted baselines include header overhead in their bit budgets. The targets specify nominal routed-weight storage. We use 512 fixed sequences of 2,048 tokens from Magpie, UltraChat, and WritingPrompts, divided into six equally weighted calibration folds~\citep{xu2025magpie,ding2023ultrachat,fan2018writingprompts}. Gradients, routing frequencies, candidate weights, and joint output errors are computed once and reused across budgets and contrast targets. Unless stated otherwise, NatScale uses joint output RMS, equally weighted structural projection factors, and $\tau=1$ nat, without an importance matrix or router fine-tuning.

\noindent\textbf{Baselines.}
We adapt five allocation methods to these shared candidates: MxMoE~\citep{duanmu2025mxmoe}, FRI-MxMoE~\citep{qi2026frimxmoe}, AlphaQ~\citep{yang2026alphaq}, MODE~\citep{chen2026mode}, and GEMQ~\citep{deng2026gemq}. They respectively emphasize reconstruction sensitivity, interpolated reconstruction costs, weight-structure priors, source-dependent expert exposure, and gradient-weighted perturbations. The GEMQ comparison evaluates its allocation criterion without GPTQ, progressive quantization, or router fine-tuning; its Qwen3.8-Flash-Next and DeepSeek-V4-Flash-0731 results use corrected gradient scoring. The adaptations are specified in \appref{app:baselines}.

\noindent\textbf{Evaluation.}
We measure perplexity (PPL), native-to-quantized KLD, and native-model top-token agreement (Top1) on WikiText-2~\citep{merity2016pointer}. Within each model, all methods use the same test stream and native reference. We score 255 next-token positions (256--510) in each independent 512-token window. Unlike the fixed-input local error estimates used to construct allocation costs, evaluation runs the complete quantized model, allowing quantization-induced changes to propagate through hidden states and routing decisions.

\subsection{Quantization Performance}
\label{sec:quantization-performance}

\noindent\textbf{Performance on WikiText-2.} Table~\ref{tab:baselines-main} shows that \method{} consistently preserves native-model predictions across architectures and storage budgets. With the fixed $\tau=1$ contrast target, it achieves the lowest KLD and highest Top1 among the adapted allocation baselines in every model--budget setting. It also yields the lowest PPL in seven of the nine settings and improves all three metrics over the adapted GEMQ allocation throughout. The improvements are especially clear on Qwen3.8-Flash-Next: at the middle budget, \method{} reduces AlphaQ's KLD from 0.2487 to 0.1941 while increasing top-token agreement from 82.82\% to 84.86\%. The results show consistent native-output preservation across the three tested architectures.

\noindent\textbf{Divergence versus Likelihood.} The PPL results qualify this advantage. AlphaQ obtains slightly lower perplexity on Qwen3.6-35B-A3B at the highest budget and on Qwen3.8-Flash-Next at the middle budget, although \method{} still achieves lower KLD and higher Top1 in both cases. This difference is expected because KLD measures agreement with the native output distribution, whereas perplexity measures likelihood assigned to observed next tokens. A quantized model can therefore more closely reproduce its native counterpart without attaining the best perplexity. These comparisons show stronger native-output fidelity at the fixed contrast target.

\noindent\textbf{Downstream task evaluation.} We further test whether the WikiText-2 gains carry over to reasoning, factuality, and coding tasks without task-specific refitting. Table~\ref{tab:natscale-downstream} reports Qwen3.8-Flash-Next at 2.5 routed bpw with thinking disabled. NatScale obtains the highest point estimate on three of four tasks: AIME 2025 (68.33\%), GPQA Diamond (78.28\%), and TruthfulQA binary (94.05\%). MODE scores highest on LiveCodeBench v6 (62.86\% versus NatScale's 61.14\%). Task rankings thus differ from native-output fidelity on WikiText-2. Evaluation protocols are given in \appref{app:downstream}.

\subsection{Ablation Studies}
\label{sec:ablation-studies}
\label{ns:threshold-experiment}
\label{sec:current-ablation}
\label{sec:main-allocation-patterns}

\noindent\textbf{Expert weight contrast.}
To isolate the effect of contrast, we vary $\tau\in\{0.5,\sqrt{0.5},1,\sqrt{2},2\}$ while retaining each model's sensitivity scores, candidate weights, local errors, and storage ceiling. Each target sets the separation of expert weights on the same information scale across models. Expert rankings remain fixed, so the sweep isolates the effect of weighting strength on allocation quality.

Figure~\ref{ns:tau-curves} shows that contrast materially affects output fidelity. On DeepSeek-V4-Flash-0731 at 2.0 bpw, increasing $\tau$ from 1 to 2 reduces KLD from 0.7293 to 0.6780 (7.03\%) and increases Top1 by 1.02 percentage points. At 3.0 bpw, however, its smallest KLD among the five targets occurs at $\tau=\sqrt{0.5}$. Across all nine settings, $\tau=2$ attains the smallest KLD in five, while $\tau=1$ yields lower PPL than $\tau=2$ in eight (\tabref{ns:tau-results}). We retain one nat as the common statistical reference across models and budgets. The sweep shows that a fixed expert ranking can lead to different quality trade-offs as contrast changes.

\begin{table*}[t]
\caption{\textbf{Whole-model Qwen3.8-Flash-Next comparison by Unsloth Dynamic (UD) reference tier.} NatScale uses each tier's UD byte ceiling. GSQ-RCO release sizes differ ($^{\dagger}$). Total size includes fixed tensors and metadata, with GB defined as $10^9$ bytes. Top1 is native-model agreement (\%). UD scores are externally reported, while NatScale and GSQ-RCO are evaluated locally using the same WikiText-2 stream and native reference. Bold indicates the lowest KLD and highest Top1 within each tier. Additional details appear in \appref{app:whole-model}.}
\label{tab:natscale-whole-main}
\centering\small

\begingroup
\setlength{\tabcolsep}{4pt}
\begin{tabular*}{\linewidth}{@{\extracolsep{\fill}}lrrrrrrrrr@{}}
\toprule
 & \multicolumn{9}{c}{UD reference tier (NatScale byte ceiling)} \\
\cmidrule(l){2-10}
Method & \multicolumn{3}{c}{\texttt{IQ1\_M}} & \multicolumn{3}{c}{\texttt{Q2\_K\_XL}} & \multicolumn{3}{c}{\texttt{IQ3\_XXS}} \\
\cmidrule(lr){2-4}\cmidrule(lr){5-7}\cmidrule(l){8-10}
 & Size (GB) & KLD $\downarrow$ & Top1 $\uparrow$ & Size (GB) & KLD $\downarrow$ & Top1 $\uparrow$ & Size (GB) & KLD $\downarrow$ & Top1 $\uparrow$ \\
\midrule
UD (reported) & 74.5 & 0.314739 & 79.691 & 78.9 & 0.224607 & 82.715 & 82.0 & 0.165120 & 85.414 \\
GSQ-RCO$^{\dagger}$ & 75.840 & 0.221337 & 82.998 & --- & --- & --- & 83.618 & 0.137169 & 86.898 \\
\midrule
NatScale (ours) & 74.529 & \textbf{0.196894} & \textbf{84.721} & 78.859 & \textbf{0.146297} & \textbf{86.543} & 81.952 & \textbf{0.124359} & \textbf{87.623} \\
\bottomrule
\end{tabular*}
\par\smallskip
{\raggedright\footnotesize
$^{\dagger}$ GSQ-RCO uses the released \texttt{IQ3\_XXS} and \texttt{IQ3\_S} checkpoints in the first and third groups, respectively. Their sizes differ from the UD reference budgets.\par}
\endgroup

\end{table*}

\begin{table}[t]
\caption{\textbf{Downstream performance on Qwen3.8-Flash-Next at 2.5 routed bpw (\%).} Thinking is disabled. AIME is mean response accuracy on AIME 2025; GPQA is GPQA Diamond; TQA is TruthfulQA binary accuracy; and LCB is LiveCodeBench v6 pass@1. GEMQ denotes the allocation-only adaptation. Bold indicates the highest point estimate per task.}
\label{tab:natscale-downstream}
\centering\small
\begingroup
\setlength{\tabcolsep}{2pt}
\begin{tabular*}{\linewidth}{@{\extracolsep{\fill}}lrrrr@{}}
\toprule
Method & AIME & GPQA & TQA & LCB \\
\midrule
MxMoE & 63.33 & 73.23 & 93.54 & 57.14 \\
FRI-MxMoE & 50.00 & 67.68 & 92.03 & 51.43 \\
AlphaQ & 62.50 & 73.74 & 93.04 & 58.86 \\
MODE & 60.83 & 74.75 & 93.42 & \textbf{62.86} \\
GEMQ (alloc.) & 15.83 & 63.13 & 93.80 & 47.43 \\
\midrule
\textbf{NatScale (ours)} & \textbf{68.33} & \textbf{78.28} & \textbf{94.05} & 61.14 \\
\bottomrule
\end{tabular*}
\endgroup

\end{table}

\noindent\textbf{Score components.}
We remove local RMS error, precision-dependent decay, gradient weights, gradient smoothing, or routing-exposure magnitude in turn. All variants reuse identical candidate weights, calibration statistics, and payload ceilings. Removing gradient weights or bit decay increases KLD and decreases Top1 in every model--budget setting; removing bit decay also worsens PPL in all nine. At 2.5 bpw, each of the five removals degrades both KLD and Top1. Removing local RMS or gradient smoothing improves all three metrics for DeepSeek-V4-Flash-0731 at 2.0 bpw, but degrades them at its higher budgets. The complete results (\appref{app:current-ablation}) support the roles of gradient weighting and precision decay.

\noindent\textbf{Joint versus separate projection errors.}
To examine the effect of projection interactions, we compare joint error measurement with the sum of independently measured gate, up, and down projection errors, keeping all other allocation conditions fixed. As shown in \tabref{tab:natscale-joint-separate-default}, neither approach consistently dominates across storage budgets. Separate scoring performs slightly better at 2.0 bpw, while higher budgets exhibit mixed results across KLD and Top1. Nevertheless, joint measurement captures the combined effects of projection perturbations that separate measurements may overlook.

We further evaluate both approaches on downstream tasks using Qwen3.8-Flash-Next at 2.5 routed bpw with thinking disabled. Joint scoring outperforms separate scoring on three of four benchmarks, with particularly notable improvements on AIME and GPQA, while separate scoring performs better on LiveCodeBench (\tabref{tab:natscale-joint-separate-downstream}). These results support joint measurement as a reasonable default for capturing projection interactions, although its benefits remain task-dependent.

\noindent\textbf{Allocation patterns.}
We visualize expert-level precision and layer-level storage at 2.5 bpw for all three models in \appref{app:allocation-comparison}. NatScale assigns different effective precisions to experts within the same layer, a variation obscured by layer-averaged budgets. The heatmaps, cumulative storage distributions, and layer-budget curves describe how the competing allocation rules distribute a global budget across experts and layers. They show how NatScale distributes precision within and across layers.

\subsection{Whole-Model Storage and Quality}
\label{sec:natscale-whole-model}

We further evaluate NatScale under whole-model storage constraints by extending precision allocation to shared experts and selected attention and gated-residual projections on Qwen3.8-Flash-Next. Unlike the routed-only experiments, this evaluation uses candidates fitted with the UD importance matrix and storage budgets aligned with Unsloth Dynamic (UD) reference tiers. We compare against UD and released GSQ-RCO checkpoints~\citep{unsloth2026flashnext,istadaslab2026flashnextgsqrco}.

Table~\ref{tab:natscale-whole-main} shows that NatScale consistently achieves lower KLD and higher Top1 agreement than the available comparators across all three storage tiers. Notably, compared with GSQ-RCO IQ3\_XXS, NatScale reduces model size by 1.31\,GB while lowering KLD by 11.04\%. These results suggest that contrast-calibrated allocation remains effective beyond routed experts, offering better preservation of native outputs under practical storage constraints. The comparisons use different inference backends, while UD results are externally reported; further protocol details are provided in \appref{app:whole-model}.

\section{Conclusion}
\label{sec:conclusion}

We presented \method{}, a mixed-precision MoE allocation framework that explicitly calibrates the contrast between expert sensitivities instead of relying on their rankings alone. It combines a model-specific, rank-preserving power transformation with joint gate, up, and down output errors and a structural precision prior to allocate weight formats under a storage budget. At the fixed one-nat target, NatScale achieves the lowest KLD and highest top-token agreement in all nine comparisons with five adapted baselines, and the lowest PPL in seven. Downstream and whole-model evaluations provide additional evidence of its utility. The contrast sweeps show that native-output fidelity and perplexity can favor different targets. %

\label{main-text-end}
\FloatBarrier
\section*{Limitations}
\label{sec:limitations}

While effective, NatScale requires additional computation to estimate expert sensitivity and select precision levels, adding to the cost of preparing a quantized model. Its allocation decisions also depend on how well the calibration data reflect the inputs encountered during deployment. Changes in input distribution or expert usage may therefore affect the quality of the resulting allocation. In addition, the error estimates used to guide precision selection are approximations and may not fully capture interactions between quantized components or changes in downstream task performance. Future work could explore more efficient calibration procedures and evaluate the stability of precision allocations across different data distributions. It could also examine whether calibration statistics can be reused when adapting a model to related application domains. Another direction is to study how the choice and amount of calibration data affect precision selection in task-specific settings.

\bibliography{references}
\clearpage
\appendix
\raggedbottom
\etocdepthtag.toc{appendix}
\begingroup
\hypersetup{linktoc=all}
\etocsettagdepth{main}{none}
\etocsettagdepth{appendix}{subsection}
\etocstandardlines
\etocsettocstyle{%
  \pdfbookmark[0]{Appendix Contents}{appendix-contents}%
  \section*{Appendix Contents}%
  \label{app:contents}%
}{}
\tableofcontents
\endgroup

\section{Contrast Calibration}
\label{ns:appendix-calibration}

This section supplement the contrast calibration introduced in Section~\ref{ns:contrast} and proves the guarantees stated in Section~\ref{ns:theory}. On the positive support $\mathcal I_+$, let $N=|\mathcal I_+|$ and assume $x_i=\log g_i$ is nonconstant. The analysis uses the log-partition function and its associated probability distribution:
\begin{equation}
 \begin{gathered}
 \psi(t)=\log\left(\frac1N\sum_i e^{tx_i}\right),
 \\ q_i^{(t)}=\frac{e^{tx_i}}{\sum_j e^{tx_j}}.
 \end{gathered}
 \label{ns:log-partition}
\end{equation}
These definitions give $\psi'(t)=\mathbb E_{q^{(t)}}x$, $\psi''(t)=\operatorname{Var}_{q^{(t)}}(x)>0$, $u=q^{(0)}$, and $q(\beta)=q^{(2\beta)}$. We write the change in expected log sensitivity as $\Delta(t)=\psi'(t)-\psi'(0)$ throughout the analysis.

\subsection{The Energy Reference}

To derive the energy reference behind the default target in Corollary~\ref{ns:weibull-corollary}, recall that the second moment $g_i$ corresponds to amplitude $r_i=\sqrt{g_i}$. After calibration, $q_i=w_i^2/\sum_jw_j^2$ describes how the corresponding energy is distributed across experts; at $\beta=1/2$, $q_i=g_i/\sum_jg_j$. The following result identifies the reference distribution for a nonnegative energy variable with a fixed mean.

\begin{proposition}[Maximum-entropy energy reference]
\label{ns:energy-reference}
Among densities of a nonnegative variable $V$ with $\mathbb E V=1$, the exponential density $e^{-v}$ uniquely maximizes differential entropy. Its amplitude $Z=\sqrt V$ has CDF
\begin{equation}
 \begin{gathered}
 F_0(z)=1-\exp(-z^2),\\ z\ge0.
 \end{gathered}
 \label{ns:energy-reference-cdf}
\end{equation}
\end{proposition}

\begin{proof}[Proof of Proposition~\ref{ns:energy-reference}]
For a density $\rho$ with $\mathbb E_\rho V=1$, define differential entropy with respect to the energy coordinate $dv$. Taking $\rho_0(v)=e^{-v}$ gives
\begin{equation}
 \begin{aligned}
 D_{\mathrm{KL}}(\rho\|\rho_0)
 &=\int_0^\infty\rho(v)\log\rho(v)\,dv\\
 &\quad+\mathbb E_\rho V\\
 &=1-h(\rho)\ge0.
 \end{aligned}
\end{equation}
Equality requires $\rho=\rho_0$ almost everywhere. Transforming $Z=\sqrt V$ gives density $2ze^{-z^2}$ and the stated CDF.
\end{proof}

Changing the positive mean rescales all amplitudes by a common factor. The entropy maximization is defined in the energy coordinate inherited from the second moment; defining it in the amplitude coordinate would give a different problem. This continuous reference provides an energy interpretation, while the finite expert distribution $q$ is calibrated separately to its prescribed contrast.

\subsection{Entropy-Regularized Reweighting}
\label{ns:gibbs-power}

To characterize this finite distribution, we express the power family as entropy-regularized reweighting. For $t\ge0$, the exponential family above is the unique solution of
\begin{equation}
 q^{(t)}=\underset{q\in\Delta_N}{\operatorname{argmax}}
 \left\{t\mathbb E_q x-D_{\mathrm{KL}}(q\|u)\right\},
 \label{ns:gibbs-variational}
\end{equation}
where $\Delta_N$ is the probability simplex. The objective equals $\psi(t)-D_{\mathrm{KL}}(q\|q^{(t)})$, so its unique maximizer is the stated distribution. Increasing $t$ strengthens the preference for larger log sensitivities. NatScale determines this preference by setting $t=2\beta_\tau$ through the realized contrast $\tau$, placing different gradient distributions on a common information scale.

\subsection{Solving for the Exponent}

This contrast-based specification determines a unique exponent because the divergence increases strictly along the power family.

\begin{proposition}[Well-defined contrast calibration]
\label{ns:unique}
If the positive gradients are nonconstant, every $\tau>0$ has a unique finite solution $\beta_\tau>0$ in \equref{ns:calibration}. This solution increases strictly with $\tau$.
\end{proposition}

\begin{proof}[Proof of Proposition~\ref{ns:unique}]
Using $\log(q_i^{(t)}/u_i)=tx_i-\psi(t)$ gives
\begin{equation}
 \mathcal J(q^{(t)},u)
 =t\Delta(t).
 \label{ns:jeffreys-proof}
\end{equation}
Since $\Delta'(t)=\psi''(t)>0$, the divergence is continuous, vanishes at $t=0$, and has derivative $\Delta(t)+t\psi''(t)>0$ for $t>0$. As $t\to\infty$, $\Delta(t)\to\max_i x_i-\mathbb E_u x>0$, so the divergence tends to infinity. Each $\tau>0$ therefore has a unique positive root, increasing with $\tau$; set $t=2\beta$.
\end{proof}

The solver uses this monotonicity to find the exponent with a bracketed root finder. It checks the support before standardization, evaluates the distribution using log-sum-exp, and verifies the returned contrast. For a fixed nonconstant score vector, the behavior near zero contrast is described by the expansion
\begin{align}
 \mathcal J(q(\beta),u)&=4\operatorname{Var}_u(x)\beta^2+O(\beta^3),\nonumber\\
 \beta_\tau&=\frac{\sqrt\tau}{2\sqrt{\operatorname{Var}_u(x)}}+O(\tau).
 \label{ns:small-tau}
\end{align}

\paragraph{Zero and degenerate statistics.}
The positive-support condition also determines how calibration handles degenerate statistics. Zero gradients retain zero weight and are excluded from calibration, while empty support or constant positive scores admit no positive-target solution and are rejected. As $\tau\downarrow0$, weights become uniform on positive support and remain zero elsewhere. When zero scores exist, this limiting distribution therefore differs from the all-one-weight ablation. All five evaluated targets are positive.

\subsection{The Weibull Case}
\label{ns:appendix-weibull}

We derive the one-nat identity in Corollary~\ref{ns:weibull-corollary} by connecting calibration to Weibull amplitude fitting. A two-parameter Weibull CDF $F(r)=1-\exp[-(r/a)^k]$ maps native amplitudes $r=\sqrt g$ to the energy reference in \equref{ns:energy-reference-cdf} as
\begin{equation}
 F_0^{-1}(F(\sqrt g))=a^{-k/2}g^{k/4}.
 \label{ns:weibull-map}
\end{equation}
The common scale leaves allocation unchanged, so the fitted shape determines the exponent. In particular, a full unweighted maximum-likelihood fit yields $\beta=\widehat k/4$, which corresponds to exactly one nat of contrast.

\begin{proof}[Proof of Corollary~\ref{ns:weibull-corollary}]
Let $z_i=\log r_i=x_i/2$ and $\bar z=N^{-1}\sum_i z_i$. For the ordinary unweighted fit to complete observations, profiling the Weibull likelihood over scale~\citep{onweibullmle} gives $\widehat a(k)^k=N^{-1}\sum_i e^{kz_i}$ and
\begin{equation}
 \begin{aligned}
 \frac{\ell_{\mathrm{prof}}(k)}N
 ={}&\log k-\log\left(\frac1N\sum_i e^{kz_i}\right)\\
 &+(k-1)\bar z-1.
 \end{aligned}
 \label{ns:weibull-profile}
\end{equation}
Its stationary equation is
\begin{equation}
 \frac1{\widehat k}
 =\sum_i\frac{e^{\widehat k z_i}}{\sum_j e^{\widehat k z_j}}z_i-\bar z.
 \label{ns:weibull-shape}
\end{equation}
The normalized profile has second derivative $-1/k^2-\operatorname{Var}_{\operatorname{softmax}(kz)}(z)<0$, so its stationary point is unique. To relate this point to contrast, observe that at $\beta=\widehat k/4$, $q=\operatorname{softmax}(\widehat k z)$. Multiplying Equation~\ref{ns:weibull-shape} by $\widehat k$ and applying Equation~\ref{ns:jeffreys-proof} then gives $\mathcal J=1$.
\end{proof}

Because the identity follows from the likelihood equation, it holds for any positive nonconstant sample without assuming that the data follow a Weibull distribution. The same argument extends to other reference coordinates: for $p>0$, mapping to $1-e^{-z^p}$ gives $w\propto g^{\widehat k/(2p)}$ and one nat for normalized $w^p$. The energy coordinate selects $p=2$, while the likelihood determines the reference contrast.

\subsection{Weight Entropy and Power Invariance}

Having calibrated the squared weights, we next examine the linear sensitivity weights used in the allocation cost of Equation~\ref{ns:complete-cost}. Their entropy remains bounded below, and their normalized values are invariant to a common positive-power reparameterization of the gradient scores.

\begin{proposition}[Entropy of the linear weights]
\label{ns:linear-entropy}
For $p_i=g_i^{\beta_\tau}/\sum_jg_j^{\beta_\tau}$,
\begin{equation}
 \begin{gathered}
 \mathcal J(p,u)<\tau/2,
 \\ H(p)>\log N-\tau/2.
 \end{gathered}
\end{equation}
\end{proposition}
\begin{proof}
Since $\Delta$ is strictly increasing,
\begin{equation}
 \mathcal J(p,u)=\beta_\tau\Delta(\beta_\tau)
 <\beta_\tau\Delta(2\beta_\tau)=\tau/2.
\end{equation}
Finally, $D_{\mathrm{KL}}(p\|u)=\log N-H(p)\le\mathcal J(p,u)$.
\end{proof}
The entropy bound gives $\exp H(p)>Ne^{-\tau/2}$ for the gradient weights before candidate-dependent costs are applied. It therefore concerns the weights themselves, not the resulting budget shares.

\begin{proposition}[Uniform positive-power reparameterization]
\label{ns:power-invariance}
If $g'_i=a g_i^c$ for common $a,c>0$, recalibration at the same $\tau$ yields $\beta'_\tau=\beta_\tau/c$ and identical normalized sensitivity weights.
\end{proposition}
\begin{proof}
Substituting $\beta'=\beta/c$ gives $q'(\beta')=q(\beta)$, so uniqueness identifies the recalibrated root. Moreover, $(g'_i)^{\beta'}=a^{\beta/c}g_i^\beta$: the transformed weights differ only by a common factor. Their normalized values, and hence the allocation when other costs are fixed, are preserved.
\end{proof}
This invariance also explains the behavior of nearly constant scores. For $g_i\propto\exp(\varepsilon z_i)$ with fixed nonconstant $z$, the exponent grows as $1/\varepsilon$ while normalized weights remain nonuniform. A fixed positive contrast thus amplifies relative differences even as $\varepsilon\downarrow0$; exactly constant scores are rejected before standardization.

\section{Residual Modeling and Sensitivity}
\label{ns:appendix-residual}

The allocation cost in Equation~\ref{ns:complete-cost} combines calibrated sensitivity weights with candidate residuals. To explain this combination, we examine raw-RMS stability and relate the gradient and error statistics of Sections~\ref{ns:contrast} and~\ref{ns:distortion} through native first-order influence.

\paragraph{Stability of raw RMS.}
We begin with the effect of small residual perturbations. For a residual vector $e\in\mathbb R^d$, let $R(e)=\|e\|_2/\sqrt d$. The reverse triangle inequality gives
\begin{equation}
 |R(e+\delta)-R(e)|\le\|\delta\|_2/\sqrt d.
 \label{ns:rms-stability}
\end{equation}
Centering the MSE before taking its square root changes this behavior. For two scalar candidates with residual $a>0$, perturbing one to $a+\delta$ changes raw RMS by $\delta$ but centered RMS by $\sqrt{2a\delta+\delta^2}$. The ratio of the centered change to the raw change diverges as $\delta\downarrow0$, showing how MSE centering amplifies near ties and makes errors depend on the candidate minimum. Shifting a \emph{complete} cost row, however, preserves allocation (Equation~\ref{ns:objective-scale}).

\paragraph{A native first-order interpretation.}
To relate residual amplitude to model sensitivity, consider a differentiable real-valued graph with locally unchanged routing indices. Interpolate each expert function as
\begin{equation}
 y_i^{(\eta)}(x)=y_i^0(x)+\eta\,[y_i^{c_i}(x)-y_i^0(x)].
\end{equation}
On native inputs, let $\epsilon_{i,t}=y_i^{c_i}(x_{i,t})-y_i^0(x_{i,t})$ and $v_{i,t}$ be the native summed-CE gradient. For $n^{\mathrm{orig}}_{s,i}>0$, define
\begin{align}
 G_{s,i}&=\frac{\sum_{t\in\mathcal T^{\mathrm{orig}}_{s,i}}\|v_{i,t}\|_2^2}{n^{\mathrm{orig}}_{s,i}H_i},\nonumber\\
 M^{\mathrm{orig}}_{s,i}&=\frac{\sum_{t\in\mathcal T^{\mathrm{orig}}_{s,i}}\|\epsilon_{i,t}\|_2^2}{n^{\mathrm{orig}}_{s,i}H_i},\nonumber\\
 I_s&=\frac1{T_s}\sum_i\sum_{t\in\mathcal T^{\mathrm{orig}}_{s,i}}v_{i,t}^{\mathsf T}\epsilon_{i,t}.
\end{align}
Zero-exposure cells contribute zero. The common-$\eta$ derivative is
\begin{equation}
 \left.\frac{d\mathcal L_{\mathrm{CE}}(\eta)}{d\eta}\right|_{\eta=0}
 =\sum_s T_s I_s.
 \label{ns:native-derivative}
\end{equation}
To incorporate fold weights, the influence $I_\pi=\sum_s\pi_s I_s$ reweights injection events by $\pi_s/T_s$ while retaining summed-CE gradients. This definition also applies when folds split interacting sequence positions. Applying Cauchy--Schwarz gives
\begin{equation}
 |I_\pi|
 \le\sum_{s,i}\pi_s H_i f_{s,i}\sqrt{G_{s,i}}\sqrt{M^{\mathrm{orig}}_{s,i}}.
 \label{ns:influence-bound}
\end{equation}
Here, the concatenated norm product divided by $T_s$ is $H_if_{s,i}\sqrt{G_{s,i}M^{\mathrm{orig}}_{s,i}}$. Because all experts within each evaluated checkpoint share an output width, $H_i$ is a common allocation factor.

This bound motivates the use of exposure, gradient amplitude, and residual amplitude in the allocation cost. Pooling gradients, calibrating weights, and applying $A_i(c)$ remain empirical choices that the bound does not determine. The implementation also distinguishes the event sets used for error measurement: $M_{s,i}$ uses separate replay events (Equation~\ref{ns:local-error}), whereas the derivative requires $M^{\mathrm{orig}}_{s,i}$ from matched native events. These sets can differ near routing boundaries.

\paragraph{Effective byte price.}
\label{ns:appendix-rate}
We now develop the storage-price interpretation from Section~\ref{ns:theory}, first for discrete candidates and then for a continuous rate model. Write $L_i(c)=A_i(c)E_i(c)$. For $w_i>0$ and a common Lagrangian byte price $\lambda\ge0$,
\begin{equation}
 \begin{aligned}
 &\argmin_c\{w_iL_i(c)+\lambda C_i(c)\}\\
 &\quad=\argmin_c\{L_i(c)+(\lambda/w_i)C_i(c)\}.
 \end{aligned}
 \label{ns:effective-price}
\end{equation}
Dividing by $w_i$ proves the identity and separates the roles of the two factors. The product $A_iE_i$ orders one expert's candidates, while a larger $w_i$ lowers that expert's effective byte price and strengthens its competition for storage.

\paragraph{Continuous high-rate model.}
To examine how this competition affects precision, suppose one continuous bit width $b_i$ describes each expert, with $M_i\simeq K_i2^{-2b_i}$, $A_i\simeq a_i2^{-2b_i}$ and $C_i=N_i b_i/8$ for $N_i$ weights. Under these assumptions, $A_i\sqrt{M_i}$ decays as $2^{-3b_i}$. Omitting explicit precision decay gives $2^{-b_i}$ for RMS or $2^{-2b_i}$ for MSE, with potentially different coefficients.

These cases can be expressed through the objective $\sum_i g_i^\beta h_i2^{-\nu b_i}$ with $g_i,h_i,N_i,\nu>0$. For an active linear budget and an interior solution, the KKT conditions give
\begin{align}
 b_i-b_j={}&\frac\beta\nu\log_2\frac{g_i}{g_j}\nonumber\\
 &+\frac1\nu\log_2\frac{h_i/N_i}{h_j/N_j}.
 \label{ns:rate-equation}
\end{align}
The relation separates sensitivity contrast $\beta$ from precision response $\nu$ in the interior model. In the implemented allocator, actual formats and projection choices instead provide discrete improvements per byte.

\section{Implementation Details}
\label{ns:implementation}

Algorithm~\ref{ns:algorithm} implements the three stages of Section~\ref{sec:method}. The details below explain how joint errors are cached and how the final allocation is rounded and certified. Preparation produces native statistics and candidate tables that can be reused across allocation settings. Changing $\tau$ requires only contrast calibration and cost construction, while changing the payload ceiling reuses these quantities as well.

\begin{algorithm*}[t]
\caption{NatScale: contrast-calibrated precision allocation}
\label{ns:algorithm}
\begin{algorithmic}[1]
\Require Native checkpoint, weighted calibration folds, candidate catalogs $\mathcal C_i$, $\tau>0$, payload ceiling $B$, optional parent map
\Ensure Static configurations $\widehat c$, payload, normalized objective and LP gap
\State Collect native gradient moments and exposure $f_{s,i}$; normalize the moments to $g_i$.
\State Solve $\mathcal J(q(\beta_\tau),u)=\tau$ on the nonconstant positive support; set $w_i=g_i^{\beta_\tau}$ there and zero elsewhere.
\State Replay fixed native inputs and cache joint $M_{s,i}(c)$ via the moment identity in \equref{ns:gram}.
\State Compute structural factors $A_i(c)$ and actual payload costs $C_i(c)$.
\State Form $d_i(c)=w_iA_i(c)\sum_s\pi_sf_{s,i}\sqrt{M_{s,i}(c)}$, setting zero-exposure terms to zero.
\State Check $\sum_i\min_c C_i(c)\le B$ and finite costs; normalize using \equref{ns:solver-cost}.
\State Build lower hulls; solve the budgeted LP and round to a feasible map as described below.
\State Keep a feasible parent map from the same catalog if its objective is smaller.
\State Verify $\sum_i C_i(\widehat c_i)\le B$; return the map, payload, objective and gap $\widetilde J(\widehat c)-\widetilde J_B^{\mathrm{LP}}$.
\end{algorithmic}
\end{algorithm*}

\paragraph{Cost normalization.}
We normalize complete cost rows using their minima, $m_i=\min_{c\in\mathcal C_i}d_i(c)$, and the global scale $S=\sum_{i\in\cI}\max_{c\in\mathcal C_i}d_i(c)$:
\begin{equation}
 \widetilde d_i(c)=\frac{d_i(c)-m_i}{S}.
 \label{ns:solver-cost}
\end{equation}
The scale must be finite and positive, and degenerate scores are rejected. Because a complete allocation selects one configuration per expert, subtracting the row minima shifts every allocation's cost by the same constant. Dividing by the positive global scale also preserves the minimizer. Both operations therefore normalize the complete allocation cost $d_i$ without changing the selected configuration; they do not center local MSE before taking its square root.

\paragraph{Calibration and candidate conventions.}
We collect native statistics from 1,048,576 chat-templated tokens divided into six folds, two per source. Gradients are measured at the down-projection output using summed next-token cross-entropy and are shared across candidates. The candidate catalogs then provide $14^3=2744$ complete configurations per expert for each Qwen model and $11^3=1331$ for DeepSeek-V4-Flash-0731, including native MXFP4.

To evaluate these candidates, we fix the inputs and routes collected by a separate native-checkpoint replay. Exposure uses $n^{\mathrm{orig}}_{s,i}/T_s$, while MSE divides the replay error sum by $H_i\max(n^{\mathrm{rep}}_{s,i},1)$. An empty replay event set therefore stores zero MSE, including when the original collection recorded positive exposure. The collector checks total route counts at the fold level; these totals can agree even when expert-level counts differ.

\paragraph{Numerical expert operator.}
Candidate evaluation follows the numerical operator of each checkpoint. Let $R_B$ denote BF16 rounding, and let the gate and up linear maps return BF16 values $a,b$. For both Qwen models, the cached activation is $R_B(R_B(\operatorname{SiLU}(a))\odot b)$. DeepSeek-V4-Flash-0731 first computes
\begin{equation}
 z=R_B\!\left(\operatorname{SiLU}(\min(a,10))
 \odot\operatorname{clip}(b,-10,10)\right),
\end{equation}
with clipping, SiLU, and multiplication in FP32 before the displayed rounding. For each 128-channel block $z_B$, its FP8 scale and reconstructed activation are
\begin{align}
 &s_B=2^{\left\lceil\log_2(\max(\|z_B\|_\infty,10^{-4})/448)\right\rceil},\nonumber\\
 &Q_8(z_B)=\nonumber\\
 &R_B\!\left(s_B R_{\mathrm{E4M3FN}}
 \bigl(\operatorname{clip}(z_B/s_B,-448,448)\bigr)\right).
\end{align}
We apply the same operator to native and candidate gate/up pairs. The moment identity in Equation~\ref{ns:gram}, derived below, scores their final down-projection matrix maps without an additional output BF16 rounding. Its cross term therefore retains both quantized activations and down-weight errors under the specified local arithmetic.

For the structural factor, our AlphaQ adaptation estimates spectra on complete $128\times128$ weight blocks, fits the largest 10\% of their spectra, and uses spectral exponent one. The variance term is the population weight variance. The precision factor uses nominal bits for NINT integer formats and effective bpw for NVQ vector formats, excluding the four-byte expert descriptor, while budget accounting uses actual candidate payload. Common positive score normalizers leave the minimizer unchanged. The adaptation uses no K-FAC attribution, projection-specific budget cap, or worst-source aggregation.

\paragraph{Moment evaluation.}
To construct the joint-error tables used in Section~\ref{ns:distortion}, we reuse activation and weight moments across candidate configurations. For one expert and fold with $n=n^{\mathrm{rep}}_{s,i}>0$ replay events, let $Z_{ab}\in\mathbb R^{n\times m}$ be gate/up activations and $D_c\in\mathbb R^{H\times m}$ a down matrix. Using checkpoint values $Z_0,D_0$ on the same replay inputs, define $\Delta Z_{ab}=Z_{ab}-Z_0$ and $\Delta D_c=D_c-D_0$. Expanding the residual $\Delta Z_{ab}D_c^{\mathsf T}+Z_0\Delta D_c^{\mathsf T}$ gives
\begin{align}
 nH M(a,b,c)={}&\langle\Delta Z_{ab}^{\mathsf T}\Delta Z_{ab},D_c^{\mathsf T}D_c\rangle_F\nonumber\\
 &+2\langle\Delta Z_{ab}^{\mathsf T}Z_0,D_c^{\mathsf T}\Delta D_c\rangle_F\nonumber\\
 &+\langle Z_0^{\mathsf T}Z_0,\Delta D_c^{\mathsf T}\Delta D_c\rangle_F.
 \label{ns:gram}
\end{align}
The identity follows from trace cyclicity and retains all off-diagonal moments and the cross term. Because the gate/up-pair and down-candidate moments are reused across triples, each configuration can be scored without replaying every token. The identity is exact for the specified local operator; we check finite-precision accumulation against direct evaluation.

This reuse changes the computational cost. With $Q$ candidates per projection and $n_{\mathrm{tot}}$ replay events, direct replay costs $O(Q^3n_{\mathrm{tot}}mH)$. For $F$ folds, evaluation with reused moments costs
\begin{equation}
 \begin{aligned}
 &O(Qn_{\mathrm{tot}}Hm)+O(Q^2n_{\mathrm{tot}}m^2)\\
 &\qquad+O(QHm^2)+O(FQ^3m^2).
 \end{aligned}
\end{equation}
The first terms account for constructing outputs and moments, while the final contractions are independent of token count. Across six retained estimates, the results agree with direct evaluation within a relative discrepancy of $4.40\times10^{-6}$. We cache one MSE for each fold, expert, and configuration for subsequent allocation.

\paragraph{LP relaxation and certificate.}
Once the candidate costs are available, we obtain a lower bound by relaxing Equation~\ref{ns:allocation} to mixtures $x_{i,c}\ge0$, $\sum_cx_{i,c}=1$, under the same budget. This LP~\citep{utboyd2004} contains every discrete allocation, so its optimum $J_B^{\mathrm{LP}}$ bounds the discrete optimum $J_B^\star$ from below. Every feasible returned configuration $\widehat c$ therefore satisfies
\begin{equation}
 0\le J(\widehat c)-J_B^\star\le J(\widehat c)-J_B^{\mathrm{LP}}.
 \label{ns:gap}
\end{equation}
The resulting certificate bounds optimization error for the fixed proxy objective. To express it in normalized cost units, let $\widetilde J(c)=\sum_i\widetilde d_i(c_i)$. Then
\begin{align}
 J(c)&=S\widetilde J(c)+\sum_i m_i,\nonumber\\
 J(\widehat c)-J_B^{\mathrm{LP}}
 &=S\bigl[\widetilde J(\widehat c)-\widetilde J_B^{\mathrm{LP}}\bigr].
 \label{ns:objective-scale}
\end{align}
Thus, the row offsets preserve both optimizers, and $S$ converts the normalized gap back to raw units.

\paragraph{Deterministic boundary rounding.}
To construct the LP solution and a feasible discrete allocation, the solver first orders each expert's candidates by ascending payload and then loss, preserving the original order for exact ties. It removes dominated points and constructs the lower convex hull, whose edges represent objective reductions obtained by spending additional bytes. The solver sorts these edges by decreasing reduction per byte. Since the hull slopes are ordered, this sorting preserves the required order of upgrades within each expert. The LP consumes complete edges until the first boundary edge exceeds the remaining budget, then takes the affordable fraction of that edge. If all upgrades fit, the endpoint allocation is already integral.

When the boundary is fractional, rounding starts from its lower-payload endpoint. The solver skips the boundary edge and scans the remaining edges in the same stable order, accepting an upgrade only when the current expert configuration is its preceding endpoint and the additional bytes fit. Each accepted move therefore preserves feasibility. If a feasible parent configuration from the same catalog has a strictly smaller objective, the solver retains it. The procedure returns a feasible map together with its measured gap to the LP optimum.

\FloatBarrier
\section{Evaluation Setup}
\label{app:baselines}

To complement the comparison in \tabref{tab:baselines-main}, this appendix describes the model architectures, shared evaluation protocol, and baseline adaptations for Qwen3.6-35B-A3B, Qwen3.8-Flash-Next, and DeepSeek-V4-Flash-0731.

\paragraph{Model scale and architecture.}
The three models span increasing reported total and activated parameter counts: 35B/3B for Qwen3.6-35B-A3B, 125B/6B for Qwen3.8-Flash-Next, and 284B/13B for DeepSeek-V4-Flash-0731~\citep{qwen2026qwen36,qwen2026flashnext,deepseek2026v4,deepseek2026flash0731}. Qwen3.6-35B-A3B has 40 layers and 256 routed experts per layer, activating eight plus one shared expert. Qwen3.8-Flash-Next has 48 layers and 512 routed experts, activating ten plus one shared expert. DeepSeek-V4-Flash-0731 has 43 layers and 256 routed experts, activating six plus one shared expert.

These reported parameter counts follow each model's specification. For Qwen3.8-Flash-Next, the reported 125B excludes the additional 51B n-gram embedding parameters and 4B multi-token-prediction (MTP) parameters. In contrast, the official 284B Flash specification includes its MTP configuration, while the 0731 checkpoint also contains the DSpark speculative-decoding module. Reported specifications, complete checkpoint parameter counts, and routed-only byte budgets therefore describe different scopes.

\paragraph{Calibration.}

For each model, we use a frozen 1,048,576-token calibration pool rendered with its native chat template and assign equal weight to the six folds. Across NatScale score variants, we reuse the native gradient second moments, routing frequencies, candidate weights, and complete joint reconstruction errors. Changing the threshold or budget therefore requires no new collection of statistics.

\paragraph{Shared candidate formats.}
All methods use candidates produced by the MixedFlowQuant (MFQ) backend. Its integer and vector-quantization formats are denoted by NINT and NVQ, respectively, with suffixes identifying the fixed candidate encodings. Qwen3.6-35B-A3B and Qwen3.8-Flash-Next use NINT2, NINT3, NINT4, NINT5, NINT6, NINT8, NVQ2J, NVQ2J-L, NVQ2J-XL, NVQ3J, NVQ3J-512, NVQ3J-L, NVQ1-S, and NVQ1-L. DeepSeek-V4-Flash-0731 uses NINT2, NINT3, the six NVQ2J/NVQ3J formats, native MXFP4, NVQ1-S, and NVQ1-L. The resulting catalogs contain 14, 14, and 11 entries, respectively. DeepSeek-V4-Flash-0731 is not assigned NINT4/5/6/8 candidates above its native MXFP4 endpoint. Within these catalogs, \method{} evaluates complete mixed gate/up/down configurations while keeping non-routed weights native.

To compare allocation methods over the same candidates, the five external baselines share these weights and the native evaluation protocol. Each baseline retains its own scoring rule and allocation constraints; the adaptations are described below in the order used in the result tables.

\paragraph{MxMoE.}
MxMoE assigns precision to linear blocks using their reconstruction sensitivity~\citep{duanmu2025mxmoe}. Our adaptation measures layer-output L2 error after a single-projection perturbation with $r=1$, ties the gate and up formats within each expert, and uses a global storage budget. Hadamard transforms and GPTQ are disabled.

\paragraph{FRI-MxMoE.}
FRI-MxMoE uses fuzzy rule interpolation to estimate reconstruction costs from a subset of profiled experts~\citep{qi2026frimxmoe}. Our adaptation uses measured expert anchors to interpolate costs across the shared candidate formats, then allocates precision under their encoded-byte costs.

\paragraph{AlphaQ.}
AlphaQ derives its allocation scores from weight spectra without calibration text~\citep{yang2026alphaq}. We adapt this criterion to the shared candidate formats by combining the spectral factor, weight variance, and bit-decay factor in Equation~\ref{ns:alphaq}. The shared candidates do not undergo the subsequent GPTQ reconstruction used in the original pipeline. The adaptation uses the released code's default spectral exponent of one,\footnote{\href{https://github.com/Superone77/AlphaQ/blob/3624976cfd800034156d4a39a3e5c04d23a02291/main.py\#L718}{AlphaQ code, commit \texttt{3624976c}}.} which differs from the model-dependent curvature rule given in the paper.

With this exponent, the adapted structural factor for projection $p$ is
\begin{equation}
 \begin{aligned}
 A_{i,p}(c_p)&=\chi_{i,p}(c_p)\frac{\widetilde\alpha}{\alpha_{i,p}}\\
 &\qquad\cdot\operatorname{Var}(W_{i,p})\,2^{-2b_{i,p}(c_p)}.
 \end{aligned}
 \label{ns:alphaq}
\end{equation}
Here $\alpha_{i,p}$ is the spectral tail index and $\widetilde\alpha$ is its median across all target projections. The precision $b_{i,p}(c_p)$ uses nominal bits for integer formats and effective bits per weight for vector formats. The indicator $\chi$ is zero only for DeepSeek-V4-Flash-0731's native MXFP4 endpoint and one otherwise. To obtain an expert-level structural factor, NatScale averages the three projection factors equally, $A_i(c)=\frac13\sum_{p\in\{G,U,D\}}A_{i,p}(c_p)$.

\paragraph{MODE.}
MODE decomposes expert importance by modality for multimodal models~\citep{chen2026mode}. Our text-only adaptation combines source-dependent exposure with expert-output reconstruction error, using the calibration text sources in place of modalities. It assigns the same candidate format to gate, up and down within each expert.

\paragraph{GEMQ (allocation-only).}
The label GEMQ (alloc.) in tables and figures denotes the allocation-only adaptation of GEMQ~\citep{deng2026gemq}. It uses the criterion in that paper's Equation~(6), combining per-token squared output gradients with whole-expert perturbations. Under this adaptation, gate, up and down share one candidate format per expert, and allocation uses encoded bytes while retaining the high-precision expert coverage constraints. GPTQ, router fine-tuning, and progressive quantization are excluded from the adapted pipeline.

To compute the gradient criterion, the evaluation protocol specifies all 512 calibration sequences with their loss masks and reuses native forward/backward statistics across budgets. We differentiate the per-sequence mean loss, with routing weights included in either the squared-gradient coefficients or the routed perturbation. The current Qwen3.8-Flash-Next and DeepSeek-V4-Flash-0731 measurements use the gradient-corrected reruns published on 2026-10-09; Qwen3.6-35B-A3B is unchanged.

\paragraph{Auxiliary weighting controls.}
Candidate fitting can also use an importance matrix (imatrix), which supplies activation-based weights. The routed-only experiments omit this weighting, whereas the whole-model extension uses the archived UD importance matrix described in \appref{app:whole-model}.

\paragraph{Native evaluation.}
We evaluate the resulting quantized models on WT2 test text without a chat template. Each window contains 512 tokens, with logits 256--510 providing 255 scored next-token predictions. The Qwen3.6-35B-A3B and Qwen3.8-Flash-Next streams contain 580 windows each, while DeepSeek-V4-Flash-0731 contains 573. Within a model, all methods use the same token stream, scored positions, and native reference. Using this common reference, KLD measures divergence from the native distribution to the quantized distribution, and Top1 measures native argmax agreement; PPL instead pools the true next-token negative log-likelihood. All three metrics are computed from complete quantized forward passes, which retain routing changes and accumulated errors.

\paragraph{Storage conventions.}
Storage budgets in these comparisons cover routed-expert weights, and the routed-bpw targets are nominal. NatScale constrains candidate payload and reports headers separately, whereas the adapted baselines reserve header bytes within their targets. For the NatScale threshold and component comparisons, we hold the payload ceiling fixed.

\FloatBarrier
\section{Precision Allocation}
\label{app:allocation-comparison}
\label{sec:precision-allocation}

The allocation patterns discussed in \secref{sec:ablation-studies} complement the quality results in \tabref{tab:baselines-main} by showing where each method assigns storage. Figures~\ref{fig:alloc-qwen-2p5}--\ref{fig:alloc-dsv-2p5} compare NatScale with MxMoE, FRI-MxMoE, AlphaQ and MODE at 2.5 nominal routed bpw. The three model figures use a common method order and color scale, while retaining each model's original expert order.

These figures present the allocations at three levels. The heatmaps show each expert's stored bits, including overhead, divided by its combined gate/up/down parameter count. To account for this overhead, tensor headers are shared equally among a layer's experts, while format and codebook headers are shared among experts using that format. The cumulative distributions then summarize storage across all experts. Finally, the layer curves show how storage varies with depth by dividing each layer's total expert storage by its routed parameter count. NatScale's variation in precision within layers is visible in the heatmaps, while the cumulative distributions show the overall spread of expert storage.

\begin{figure*}[!t]
\centering
\includegraphics[width=\linewidth]{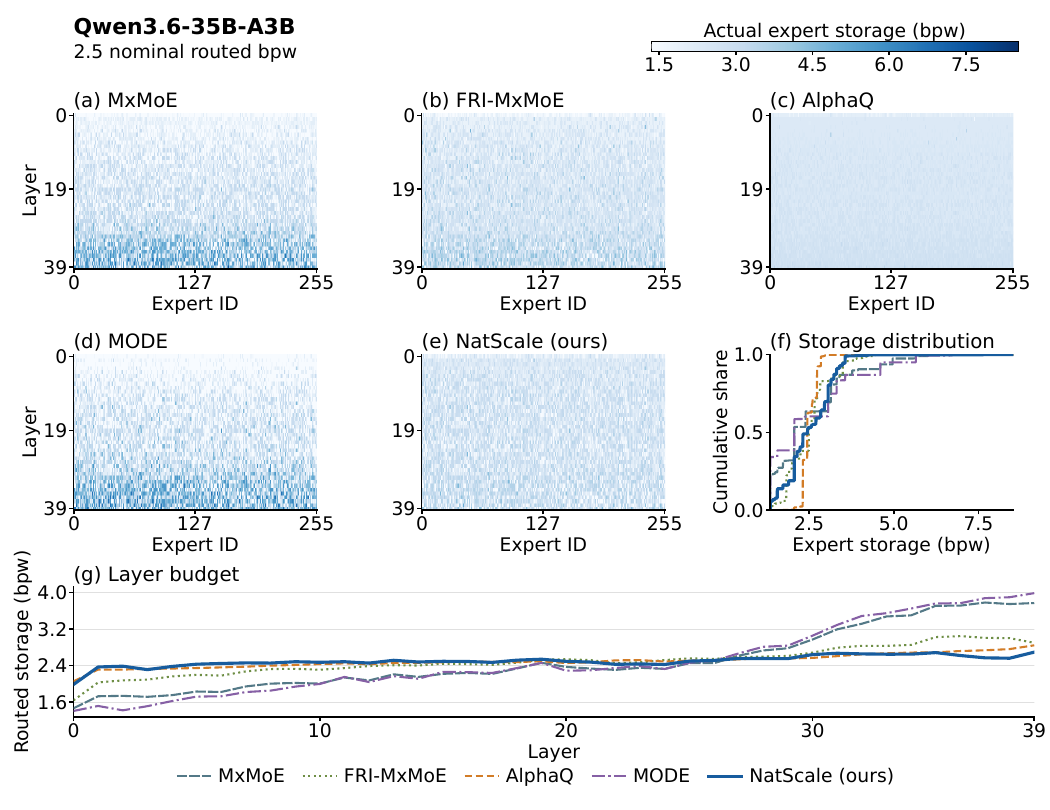}
\caption{Expert precision and layer budgets for Qwen3.6-35B-A3B at nominal 2.5 routed bpw. Panels (a)--(e) show actual expert storage; (f) gives its cumulative distribution; (g) shows routed storage by layer. The five methods use the same order, storage definition and color scale across the three model figures.}
\label{fig:alloc-qwen-2p5}
\end{figure*}

\begin{figure*}[!t]
\centering
\includegraphics[width=\linewidth]{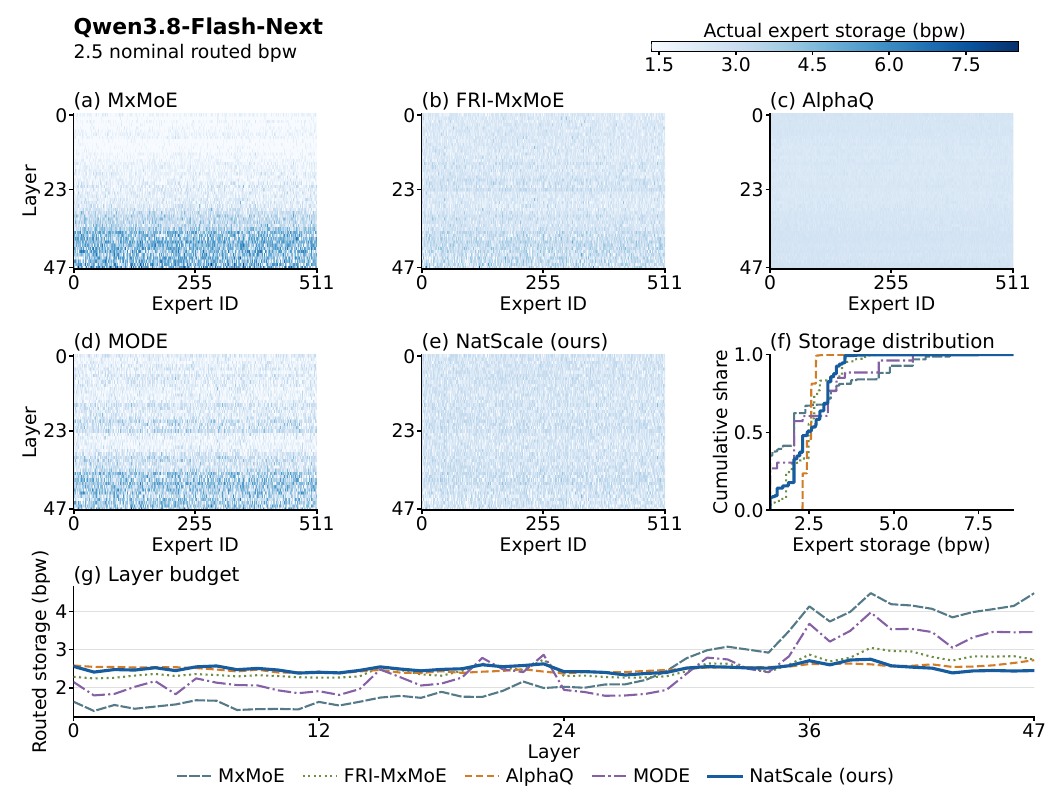}
\caption{Expert precision and layer budgets for Qwen3.8-Flash-Next at nominal 2.5 routed bpw. Heatmaps, cumulative storage distributions and layer curves follow the definitions and shared scales in Figure~\ref{fig:alloc-qwen-2p5}. NatScale uses the joint-error default.}
\label{fig:alloc-flash-2p5}
\end{figure*}

\begin{figure*}[!t]
\centering
\includegraphics[width=\linewidth]{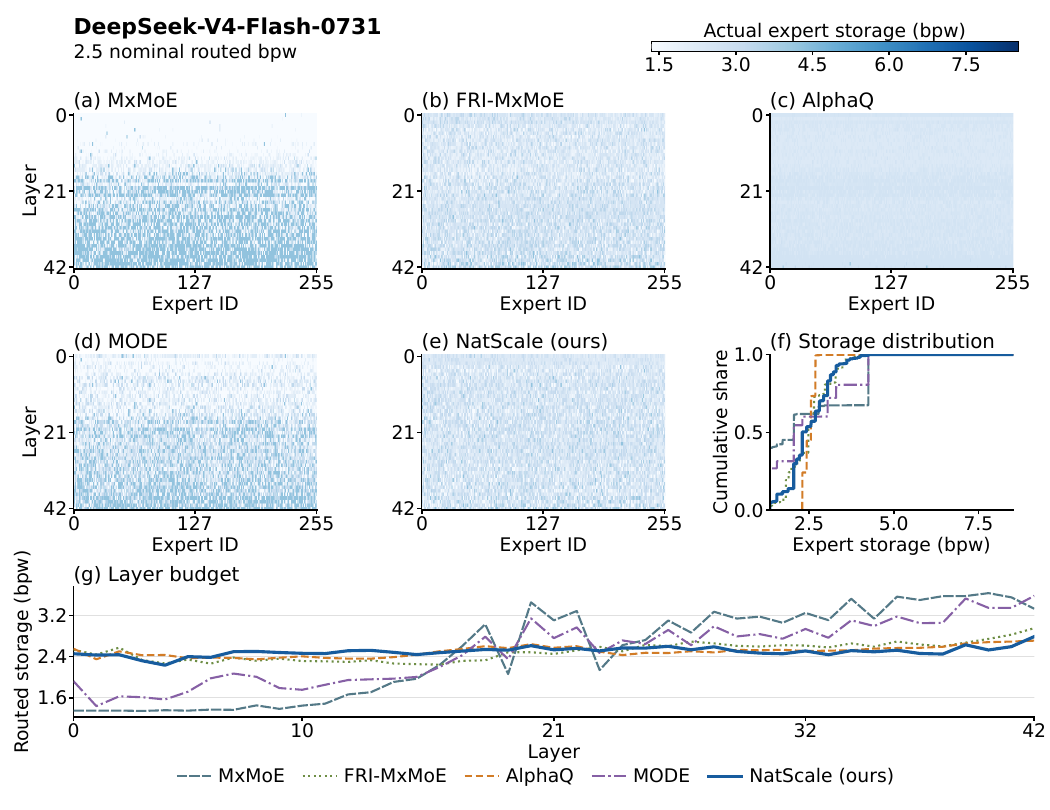}
\caption{Expert precision and layer budgets for DeepSeek-V4-Flash-0731 at nominal 2.5 routed bpw. All 43 layers and 256 routed experts per layer are shown in their original order. Heatmaps, cumulative distributions and layer curves use the same definitions and scales as the two Qwen figures.}
\label{fig:alloc-dsv-2p5}
\end{figure*}

\section{Design Ablations}
\label{app:current-ablation}

This appendix provides the protocols and complete results behind the design comparisons in \secref{sec:ablation-studies}. We first vary the contrast target and score factors across three models and budgets, then compare joint and separate projection errors on Qwen3.8-Flash-Next. Throughout these comparisons, calibration statistics and candidate weights remain fixed, as does each setting's exact payload ceiling.

\subsection{Contrast Targets}
\label{ns:threshold-appendix}

We begin with the calibration rule in \equref{ns:calibration}, varying its contrast target around the one-nat reference derived in \secref{sec:theory}. For each model, we solve the Jeffreys calibration condition at five log-spaced targets $\tau\in\{0.5,\sqrt{0.5},1,\sqrt{2},2\}$ nats on the positive-gradient support, using natural logarithms and no factor of $1/2$. The resulting exponent enters the allocation score through $g^{\beta_\tau}$, whereas the normalized distribution used in the divergence is proportional to $g^{2\beta_\tau}$. At each target, one exponent is shared across layers and budgets, and zero-gradient experts retain zero weight. For the one-nat setting, we use the saved Weibull-MLE exponent exactly.

This comparison changes only the gradient exponent: raw joint RMS, the structural factor that equally weights the three AlphaQ projection factors, fold frequencies, candidate weights, and payload ceilings remain fixed. In the table, the labels 0.7071 and 1.4142 denote the exact midpoints $\sqrt{0.5}$ and $\sqrt{2}$. \tabref{ns:tau-results} provides the KLD and Top1 values underlying Figure~\ref{ns:tau-curves} and additionally reports PPL, showing how preferences for contrast differ between likelihood and native-output fidelity.

\begin{table*}[!t]
\caption{All 45 threshold settings, including the nine defaults. Top1 is native-reference agreement (\%). Bold marks each metric's best value among the five contrasts within each model and budget.}
\label{ns:tau-results}
\centering\small
\begingroup
\setlength{\tabcolsep}{3pt}
\begin{tabular*}{\linewidth}{@{\extracolsep{\fill}}lrrrrr@{}}
\toprule
Model & Target bpw & $\tau$ (nat) & KLD $\downarrow$ & Top1 (\%) $\uparrow$ & PPL $\downarrow$ \\
\midrule
\shortstack[l]{Qwen3.6-\\35B-A3B} & 2.0 & 0.5 & 0.202861 & 82.8668 & 7.703205 \\
 &  & 0.7071 & 0.201441 & 82.9946 & 7.682874 \\
 &  & 1 & 0.198816 & 82.9824 & 7.676734 \\
 &  & 1.4142 & 0.196132 & \textbf{83.1650} & \textbf{7.676333} \\
 &  & 2 & \textbf{0.195436} & 83.0852 & 7.689436 \\
\addlinespace[2pt]
 & 2.5 & 0.5 & 0.100756 & 87.7958 & \textbf{7.125175} \\
 &  & 0.7071 & 0.099959 & 87.8188 & 7.154521 \\
 &  & 1 & \textbf{0.098355} & \textbf{87.9378} & 7.150688 \\
 &  & 1.4142 & 0.099830 & 87.8431 & 7.204622 \\
 &  & 2 & 0.100853 & 87.8580 & 7.248561 \\
\addlinespace[2pt]
 & 3.0 & 0.5 & 0.058600 & 90.6437 & 7.037591 \\
 &  & 0.7071 & 0.057835 & 90.6876 & 7.025139 \\
 &  & 1 & 0.057493 & 90.7228 & \textbf{7.014946} \\
 &  & 1.4142 & 0.057376 & \textbf{90.7897} & 7.022438 \\
 &  & 2 & \textbf{0.057317} & 90.7674 & 7.042827 \\
\addlinespace[2pt]
\midrule
\shortstack[l]{Qwen3.8-\\Flash-Next} & 2.0 & 0.5 & 0.345286 & 80.1237 & 5.079686 \\
 &  & 0.7071 & 0.337696 & \textbf{80.3198} & 5.074160 \\
 &  & 1 & \textbf{0.329276} & 80.2170 & \textbf{5.067797} \\
 &  & 1.4142 & 0.336225 & 80.1034 & 5.073389 \\
 &  & 2 & 0.340955 & 80.0169 & 5.086687 \\
\addlinespace[2pt]
 & 2.5 & 0.5 & 0.201418 & 84.7796 & 4.823743 \\
 &  & 0.7071 & 0.197665 & 84.8621 & \textbf{4.822990} \\
 &  & 1 & \textbf{0.194086} & 84.8573 & 4.834395 \\
 &  & 1.4142 & 0.194985 & \textbf{84.8763} & 4.851031 \\
 &  & 2 & 0.196514 & 84.6606 & 4.878319 \\
\addlinespace[2pt]
 & 3.0 & 0.5 & 0.127064 & 87.6876 & \textbf{4.736719} \\
 &  & 0.7071 & 0.125395 & 87.7735 & 4.740789 \\
 &  & 1 & 0.124340 & 87.6714 & 4.743009 \\
 &  & 1.4142 & 0.121826 & 87.7437 & 4.749630 \\
 &  & 2 & \textbf{0.121556} & \textbf{87.7830} & 4.748906 \\
\addlinespace[2pt]
\midrule
\shortstack[l]{DeepSeek-V4-\\Flash-0731} & 2.0 & 0.5 & 0.765042 & 71.3554 & 8.648208 \\
 &  & 0.7071 & 0.751667 & 71.6217 & 8.579417 \\
 &  & 1 & 0.729311 & 72.1534 & 8.364731 \\
 &  & 1.4142 & 0.701516 & 72.6845 & 8.144471 \\
 &  & 2 & \textbf{0.678030} & \textbf{73.1780} & \textbf{7.925372} \\
\addlinespace[2pt]
 & 2.5 & 0.5 & 0.346255 & 81.1532 & 5.978910 \\
 &  & 0.7071 & 0.343613 & 81.2730 & 6.008364 \\
 &  & 1 & 0.336141 & 81.5714 & 5.961485 \\
 &  & 1.4142 & 0.332286 & 81.6555 & \textbf{5.939674} \\
 &  & 2 & \textbf{0.331466} & \textbf{81.7363} & 5.967101 \\
\addlinespace[2pt]
 & 3.0 & 0.5 & 0.193404 & 86.1705 & \textbf{5.473308} \\
 &  & 0.7071 & \textbf{0.192386} & 86.2376 & 5.477620 \\
 &  & 1 & 0.193588 & 86.1609 & 5.488280 \\
 &  & 1.4142 & 0.194685 & \textbf{86.2430} & 5.492832 \\
 &  & 2 & 0.196102 & 86.1883 & 5.509159 \\
\addlinespace[2pt]
\bottomrule
\end{tabular*}
\endgroup

\end{table*}

\subsection{Component Ablations}
\label{ns:full-component-results}

Having varied the contrast target, we next examine the factors combined in \equref{ns:complete-cost} through five core removals from the default, which uses joint raw RMS and $\tau=1$. Removing local RMS error sets the error factor to one. When bit decay is removed, the spectral and variance factors remain, and DeepSeek-V4-Flash-0731's native MXFP4 endpoint retains zero added error.

The remaining three removals concern expert weighting. Removing gradient weights assigns every expert a weight of one, including experts with zero native gradient; removing gradient smoothing instead sets $\beta=1$. Finally, removing exposure magnitude replaces the frequency $f_{s,i}$ with $\mathbf{1}[f_{s,i}>0]$, while preserving the routing dependence of the gradient statistics.

To examine the structural factors and error power in more detail, we introduce four additional controls for Qwen3.8-Flash-Next. The MSE variant also removes bit decay and therefore differs from the default in two factors. Comparing it with the RMS variant without bit decay isolates the effect of error power under that condition. The other three controls retain bit decay while removing the spectral factor, the variance factor, or both. Tables~\ref{tab:natscale-ablation-qwen}--\ref{tab:natscale-ablation-dsv} report 66 component results across KLD, Top1 and PPL.

\begin{table*}[!tp]
\caption{NatScale component variants for Qwen3.6-35B-A3B under fixed candidate weights and payload ceilings. Top1 is native agreement (\%). Bold marks the best listed variant within each budget and metric; the joint-error default appears last.}
\label{tab:natscale-ablation-qwen}
\centering\small
\begingroup
\setlength{\tabcolsep}{4pt}
\begin{tabular*}{\linewidth}{@{\extracolsep{\fill}}rlrrr@{}}
\toprule
bpw & Variant & KLD $\downarrow$ & Top1 (\%) $\uparrow$ & PPL $\downarrow$ \\
\midrule
2.0 & NatScale w/o local RMS error & 0.243853 & 81.1204 & 8.156114 \\
2.0 & NatScale w/o bit decay & 0.310342 & 79.1062 & 8.450841 \\
2.0 & NatScale w/o gradient weights & 0.222994 & 81.9155 & 7.854326 \\
2.0 & NatScale w/o gradient smoothing & 0.245525 & 81.0257 & 8.210989 \\
2.0 & NatScale w/o exposure magnitude & \textbf{0.194500} & 82.4719 & 7.719427 \\
\cmidrule(lr){1-5}
2.0 & NatScale (ours, $\tau=1$) & 0.198816 & \textbf{82.9824} & \textbf{7.676734} \\
\midrule
2.5 & NatScale w/o local RMS error & 0.125210 & 86.9506 & 7.445417 \\
2.5 & NatScale w/o bit decay & 0.184559 & 84.4388 & 7.676710 \\
2.5 & NatScale w/o gradient weights & 0.105059 & 87.5808 & 7.185638 \\
2.5 & NatScale w/o gradient smoothing & 0.131737 & 86.2691 & 7.474272 \\
2.5 & NatScale w/o exposure magnitude & 0.106081 & 86.8796 & 7.207796 \\
\cmidrule(lr){1-5}
2.5 & NatScale (ours, $\tau=1$) & \textbf{0.098355} & \textbf{87.9378} & \textbf{7.150688} \\
\midrule
3.0 & NatScale w/o local RMS error & 0.068296 & 90.0609 & 7.131690 \\
3.0 & NatScale w/o bit decay & 0.121408 & 87.6362 & 7.296315 \\
3.0 & NatScale w/o gradient weights & 0.061691 & 90.4320 & 7.016817 \\
3.0 & NatScale w/o gradient smoothing & 0.075514 & 89.5612 & 7.201615 \\
3.0 & NatScale w/o exposure magnitude & 0.059845 & 90.1657 & \textbf{7.014259} \\
\cmidrule(lr){1-5}
3.0 & NatScale (ours, $\tau=1$) & \textbf{0.057493} & \textbf{90.7228} & 7.014946 \\
\bottomrule
\end{tabular*}
\endgroup

\end{table*}
\begin{table*}[!tp]
\caption{NatScale component variants for Qwen3.8-Flash-Next, including error-power and spectral/variance controls. Bold marks the best variant within each budget and metric.}
\label{tab:natscale-ablation-flash}
\centering\small
\begingroup
\setlength{\tabcolsep}{4pt}
\begin{tabular*}{\linewidth}{@{\extracolsep{\fill}}rlrrr@{}}
\toprule
bpw & Variant & KLD $\downarrow$ & Top1 (\%) $\uparrow$ & PPL $\downarrow$ \\
\midrule
2.0 & NatScale w/o local RMS error & 0.378134 & 78.5598 & 5.302109 \\
2.0 & NatScale w/o bit decay & 0.456162 & 76.5707 & 5.582535 \\
2.0 & NatScale w/o gradient weights & 0.337510 & 79.8438 & 5.154556 \\
2.0 & NatScale w/o gradient smoothing & 0.492273 & 75.1075 & 5.816586 \\
2.0 & NatScale w/o exposure magnitude & 0.330178 & 80.1420 & \textbf{5.035539} \\
2.0 & NatScale w/ MSE, w/o bit decay & 0.389347 & 78.4050 & 5.372212 \\
2.0 & NatScale w/o spectral factor & \textbf{0.329165} & 80.1866 & 5.086408 \\
2.0 & NatScale w/o variance factor & 0.332328 & 80.1501 & 5.082415 \\
2.0 & NatScale w/o spectral and variance factors & 0.334005 & 80.0710 & 5.081566 \\
\cmidrule(lr){1-5}
2.0 & NatScale (ours, $\tau=1$) & 0.329276 & \textbf{80.2170} & 5.067797 \\
\midrule
2.5 & NatScale w/o local RMS error & 0.226839 & 83.6518 & 4.909631 \\
2.5 & NatScale w/o bit decay & 0.292266 & 81.6971 & 5.072827 \\
2.5 & NatScale w/o gradient weights & 0.202779 & 84.6815 & \textbf{4.802855} \\
2.5 & NatScale w/o gradient smoothing & 0.314415 & 80.3550 & 5.236243 \\
2.5 & NatScale w/o exposure magnitude & 0.223295 & 83.6180 & 4.858623 \\
2.5 & NatScale w/ MSE, w/o bit decay & 0.238135 & 83.3698 & 4.925901 \\
2.5 & NatScale w/o spectral factor & 0.196914 & \textbf{84.9601} & 4.829930 \\
2.5 & NatScale w/o variance factor & 0.199316 & 84.8431 & 4.830544 \\
2.5 & NatScale w/o spectral and variance factors & 0.197706 & 84.9020 & 4.835880 \\
\cmidrule(lr){1-5}
2.5 & NatScale (ours, $\tau=1$) & \textbf{0.194086} & 84.8573 & 4.834395 \\
\midrule
3.0 & NatScale w/o local RMS error & 0.138895 & 87.1481 & 4.822337 \\
3.0 & NatScale w/o bit decay & 0.202124 & 84.8722 & 4.902857 \\
3.0 & NatScale w/o gradient weights & 0.128369 & 87.6504 & \textbf{4.736839} \\
3.0 & NatScale w/o gradient smoothing & 0.197610 & 84.6646 & 4.956623 \\
3.0 & NatScale w/o exposure magnitude & 0.130632 & 87.3597 & 4.758108 \\
3.0 & NatScale w/ MSE, w/o bit decay & 0.138096 & 86.9229 & 4.825457 \\
3.0 & NatScale w/o spectral factor & \textbf{0.123582} & \textbf{87.7458} & 4.737910 \\
3.0 & NatScale w/o variance factor & 0.125899 & 87.7363 & 4.737174 \\
3.0 & NatScale w/o spectral and variance factors & 0.123974 & 87.6788 & 4.755737 \\
\cmidrule(lr){1-5}
3.0 & NatScale (ours, $\tau=1$) & 0.124340 & 87.6714 & 4.743009 \\
\bottomrule
\end{tabular*}
\endgroup

\end{table*}
\begin{table*}[!t]
\caption{NatScale component variants for DeepSeek-V4-Flash-0731. Bold marks the best variant within each budget and metric.}
\label{tab:natscale-ablation-dsv}
\centering\small
\begingroup
\setlength{\tabcolsep}{4pt}
\begin{tabular*}{\linewidth}{@{\extracolsep{\fill}}rlrrr@{}}
\toprule
bpw & Variant & KLD $\downarrow$ & Top1 (\%) $\uparrow$ & PPL $\downarrow$ \\
\midrule
2.0 & NatScale w/o local RMS error & 0.655823 & 73.9828 & 7.726484 \\
2.0 & NatScale w/o bit decay & 1.102132 & 65.5826 & 11.946395 \\
2.0 & NatScale w/o gradient weights & 0.912658 & 69.1134 & 10.175868 \\
2.0 & NatScale w/o gradient smoothing & \textbf{0.609329} & \textbf{74.8965} & \textbf{7.378263} \\
2.0 & NatScale w/o exposure magnitude & 0.787991 & 71.1932 & 8.884293 \\
\cmidrule(lr){1-5}
2.0 & NatScale (ours, $\tau=1$) & 0.729311 & 72.1534 & 8.364731 \\
\midrule
2.5 & NatScale w/o local RMS error & 0.391052 & 80.3449 & 6.339849 \\
2.5 & NatScale w/o bit decay & 0.639630 & 73.7433 & 7.553834 \\
2.5 & NatScale w/o gradient weights & 0.401876 & 79.5743 & 6.266410 \\
2.5 & NatScale w/o gradient smoothing & 0.367891 & 81.3003 & 6.207998 \\
2.5 & NatScale w/o exposure magnitude & 0.381792 & 80.2033 & 6.222373 \\
\cmidrule(lr){1-5}
2.5 & NatScale (ours, $\tau=1$) & \textbf{0.336141} & \textbf{81.5714} & \textbf{5.961485} \\
\midrule
3.0 & NatScale w/o local RMS error & 0.242478 & 84.7367 & 5.720269 \\
3.0 & NatScale w/o bit decay & 0.295438 & 82.9395 & 5.789639 \\
3.0 & NatScale w/o gradient weights & 0.198781 & 85.8331 & 5.530394 \\
3.0 & NatScale w/o gradient smoothing & 0.217450 & 85.6435 & 5.595068 \\
3.0 & NatScale w/o exposure magnitude & 0.211537 & 85.4875 & 5.545926 \\
\cmidrule(lr){1-5}
3.0 & NatScale (ours, $\tau=1$) & \textbf{0.193588} & \textbf{86.1609} & \textbf{5.488280} \\
\bottomrule
\end{tabular*}
\endgroup

\end{table*}

Across the nine model--budget settings, removing bit decay consistently worsens all three metrics. Gradient weights also improve native-output fidelity throughout: removing them worsens KLD and Top1 in every setting, although it improves PPL in two. By comparison, the effects of local RMS error and gradient smoothing depend on the setting. Removing either improves all three metrics for DeepSeek-V4-Flash-0731 at 2.0 bpw but worsens them in the other eight settings.

The exposure and spectral controls further distinguish the metrics. Removing exposure magnitude improves KLD once and PPL twice, while never improving Top1. On Qwen3.8-Flash-Next, removing the spectral factor improves all three metrics at 3.0 bpw, whereas at 2.5 bpw it trades worse KLD for better Top1 and PPL.

\subsection{Joint and Separate Errors}
\label{ns:joint-separate-error}

The joint-error design in \secref{ns:distortion} retains interactions among gate, up and down perturbations. To evaluate this choice in the ablation reported in \secref{sec:ablation-studies}, we replace the joint MSE in \equref{ns:local-error} with a sum of single-projection perturbation errors:
\begingroup
\begin{equation}
 \begin{aligned}
 M^{\mathrm{sep}}_{s,i}(c_G,c_U,c_D)
 ={}&M_{s,i}(c_G,0,0)\\
 &+M_{s,i}(0,c_U,0)\\
 &+M_{s,i}(0,0,c_D).
 \end{aligned}
 \label{ns:separate-error-ablation}
\end{equation}
\endgroup
Here $0$ denotes the native projection, used only as a measurement reference. Each term changes one projection while keeping the other two native. To form the separate-error score, we sum these MSEs before taking the square root, while retaining the calibrated gradients, structural factors, exposure weights, candidate weights and global payload ceilings. Both variants allow G/U/D to select different formats; joint scoring remains the default, and separate-error scoring serves as the control.

This change in measurement reduces the number of local-error entries from $Q^3$ to $3Q$ per expert and fold, or from 2744 to 42 for Qwen3.8-Flash-Next. It leaves the same $Q^3$ deployment configurations available to both scores. The joint implementation evaluates its entries by reusing matrix moments.

\begin{table*}[!t]
\caption{Joint versus separate projection errors on Qwen3.8-Flash-Next at the NatScale default. Both variants use raw RMS and $\tau=1$ under each setting's fixed payload ceiling. Top1 is native agreement (\%). Bold marks the better of the two scores within each row and metric, using unrounded values.}
\label{tab:natscale-joint-separate-default}
\centering\small
\begingroup
\setlength{\tabcolsep}{3pt}
\begin{tabular*}{\linewidth}{@{\extracolsep{\fill}}lrrrrrr@{}}
\toprule
Setting & \multicolumn{3}{c}{Joint error} & \multicolumn{3}{c}{Separate error} \\
\cmidrule(lr){2-4}\cmidrule(l){5-7}
 & KLD $\downarrow$ & Top1 (\%) $\uparrow$ & PPL $\downarrow$ & KLD $\downarrow$ & Top1 (\%) $\uparrow$ & PPL $\downarrow$ \\
\midrule
\multicolumn{7}{l}{\textbf{Qwen3.8-Flash-Next}}\\
2.0 bpw & 0.329276 & 80.2170 & \textbf{5.067797} & \textbf{0.326225} & \textbf{80.2569} & 5.073339 \\
2.5 bpw & \textbf{0.194086} & 84.8573 & 4.834395 & 0.195639 & \textbf{84.9175} & \textbf{4.831680} \\
3.0 bpw & 0.124340 & \textbf{87.6714} & 4.743009 & \textbf{0.123854} & 87.6423 & \textbf{4.741410} \\
\bottomrule
\end{tabular*}
\endgroup

\end{table*}

The resulting comparison on Qwen3.8-Flash-Next shows that separate scoring lowers KLD at 2.0 and 3.0 bpw and raises Top1 at 2.0 and 2.5 bpw. At 2.5 bpw, for example, it raises Top1 from 84.8573\% to 84.9175\% but increases KLD from 0.194086 to 0.195639. At 3.0 bpw, the metric preferences reverse: separate scoring has lower KLD, while joint scoring has higher Top1.

To assess how these error scores affect task performance, the downstream comparison in \tabref{tab:natscale-joint-separate-downstream} evaluates them on four tasks using Qwen3.8-Flash-Next at 2.5 routed bpw.

\section{Downstream Evaluation}
\label{app:downstream}

\secref{sec:quantization-performance} asks whether the gains in native-output fidelity carry over to downstream tasks. This appendix specifies the protocols for that comparison in \tabref{tab:natscale-downstream}, which evaluates NatScale and five allocation baselines on Qwen3.8-Flash-Next at nominal 2.5 routed bpw. Using the same model and storage target, the error-score ablation in \tabref{tab:natscale-joint-separate-downstream} further examines the choice of error estimate. The default NatScale configuration uses joint raw RMS and $\tau=1$, while the ablation replaces the joint estimate with the separate errors defined in \equref{ns:separate-error-ablation}.

\paragraph{Shared inference conditions.}
For both comparisons, we use a modified Benchi pipeline with thinking disabled. The pipeline reuses the fixed allocations and candidate weights from WT2 evaluation without task-specific refitting. Non-routed weights remain native and candidate fitting uses no imatrix. The five baselines use the shared-MFQ adaptations described in \appref{app:baselines}. For AIME and GPQA, evaluation uses the corrected chat-completions interface and native chat template; earlier completion-interface runs are excluded.

\paragraph{AIME 2025.}
For AIME, we generate four responses to each of 30 problems at temperature 1, top-$p$ 1 and top-$k$ 20, with a 32,768-token limit. MathArena~\citep{balunovic2025matharena} then scores mean accuracy across all generated responses. For joint-error NatScale on Qwen3.8-Flash-Next, we exclude the faulty run and use the corrected runs with seeds 22345 and 32345. Each corrected run provides two responses per problem, giving a combined score of 82/120 (68.33\%). The separate-error runs use seed 12345 to generate all four responses per problem.

\paragraph{GPQA Diamond.}
For GPQA, all methods are evaluated on 198 questions~\citep{rein2023gpqa} with greedy decoding and a 32,768-token limit. Each method skips the same long question, which is counted as incorrect. Thus, evaluation requires 197 requests but retains a denominator of 198. Refusals remain in the denominator.

\paragraph{TruthfulQA binary.}
We use the binary-2025 adaptation of TruthfulQA~\citep{lin2022truthfulqa}, which contains 790 questions. Responses are generated with greedy decoding and a 32-token limit, while seed 0 fixes the answer-option order. Accuracy is the fraction of questions answered correctly.

\paragraph{LiveCodeBench v6.}
We evaluate the 175 newly added v6 problems in \texttt{test6.jsonl}~\citep{jain2025livecodebench}, using one greedy response per problem, seed 0 and a 32,768-token limit. The generated responses are evaluated by a checker with a six-second timeout and 12 scoring processes; pass@1 uses the first valid grade for each problem. Where task-index correspondence requires repair, we restore it and generate only the missing answers. Every reported configuration therefore covers all 175 unique problems.

\paragraph{Joint and separate errors.}
\label{app:downstream-error-ablation}

With these protocols fixed, \tabref{tab:natscale-joint-separate-downstream} extends the WikiText-2 error comparison in \secref{sec:ablation-studies} to task performance. The separate-error scores for Qwen3.8-Flash-Next come from the 12 October export, while the joint-error scores are those reported in the main table. Both follow the AIME protocol described above, with the joint-error result taken from the corrected runs.

\begin{table*}[!t]
\caption{Downstream comparison of joint and separate-error NatScale on Qwen3.8-Flash-Next at nominal 2.5 routed bpw, with thinking disabled. Cells show percentages and correct counts over the full denominator. AIME is mean response accuracy on AIME 2025; GPQA is Diamond; TQA is TruthfulQA binary; LCB is LiveCodeBench v6 pass@1. Bold marks the higher point estimate on each task.}
\label{tab:natscale-joint-separate-downstream}
\centering\small
\begingroup
\setlength{\tabcolsep}{4pt}
\begin{tabular*}{\linewidth}{@{\extracolsep{\fill}}lrrrr@{}}
\toprule
Error score & AIME & GPQA & TQA & LCB \\
\midrule
\multicolumn{5}{@{}l}{\textbf{Qwen3.8-Flash-Next}} \\
NatScale w/ separate errors & 65.00\,(78/120) & 74.75\,(148/198) & 93.42\,(738/790) & \textbf{62.86}\,(110/175) \\
\cmidrule(lr){1-5}
NatScale (joint, ours) & \textbf{68.33}\,(82/120) & \textbf{78.28}\,(155/198) & \textbf{94.05}\,(743/790) & 61.14\,(107/175) \\
\bottomrule
\end{tabular*}
\endgroup

\end{table*}

Joint scoring achieves higher scores on AIME, GPQA and TruthfulQA for Qwen3.8-Flash-Next, whereas separate scoring performs better on LiveCodeBench. We retain joint scoring as the default because it preserves projection interactions.

\section{Whole-Model Allocation}
\label{app:whole-model}

The whole-model comparison on Qwen3.8-Flash-Next in \secref{sec:natscale-whole-model} tests whether NatScale's allocation rule remains effective when shared experts, attention and gated-residual projections also receive quantized formats. To define this extension, we first generalize the expert cost in \equref{ns:complete-cost} to these allocation units, then include fixed tensors and metadata in the global byte constraint. The gated-residual projections are archived as MHC, and the identifier \texttt{G1\_raw} denotes raw-RMS scoring under this whole-model protocol.

\subsection{Allocation Units and Scores}

To define allocation units for this extension, we follow the checkpoint's modular structure, which combines experts with Qwen Sparse Attention (QSA), Gated DeltaNet (GDN) and gated-residual modules~\citep{qwen2026flashnext}. The resulting set $\mathcal U$ contains 24,576 routed experts, 48 shared experts, 48 attention groups and 96 MHC groups, totaling 24,768 units and 74,304 projections. Each expert groups its gate/up/down projections. Attention groups contain Q/K/V/G/O in QSA layers or Q/K/V/Z/O in GDN layers, while MHC groups contain up/down. Within a unit, projections may receive different encodings from the 14-format catalog (\appref{app:baselines}). Across all three budgets, we reuse candidates fitted with the UD importance matrix~\citep{unsloth2026flashnext} and compute allocation statistics from our native calibration data.

We next score each unit's candidate configurations. For a unit $u$, let $P_u$ denote its projections and $c$ their joint configuration. Routed and shared experts retain the complete joint G/U/D MSE $M_{s,u}(c)$, including cross terms, measured on fixed native inputs. For attention and MHC, we instead sum the single-projection module-output MSEs $M_{s,u,p}(c_p)$, omitting cross terms between their projection perturbations. Averaging the AlphaQ structural factors with equal shares gives $\overline A_u(c)=|P_u|^{-1}\sum_{p\in P_u}A_{u,p}(c_p)$. The resulting score is
\begin{equation}
\begin{gathered}
 \widetilde M_{s,u}(c)=\\
 \begin{cases}
 M_{s,u}(c),&\text{routed/shared expert},\\
 \sum_{p\in P_u}M_{s,u,p}(c_p),&\text{attention/MHC},
 \end{cases}\\
 d_u(c)=\\g_u^\beta\overline A_u(c)
 \sum_{s=1}^{6}\pi_s f_{s,u}\sqrt{\max\{\widetilde M_{s,u}(c),0\}}.
\end{gathered}
\label{ns:whole-model-score}
\end{equation}
Here, $\pi_s=1/6$ gives equal weight to each fold, while $f_{s,u}$ is the native activation frequency for routed experts and one otherwise. The zero floor handles numerical roundoff without subtracting any candidate minimum. For attention and MHC, the single-projection MSEs are first added, then converted to RMS by taking the square root. As with the joint expert error, this square root is taken before averaging across folds.

To obtain the sensitivity weights in this score, we compute one conditional second moment per unit from native output gradients of the summed cross-entropy loss, then normalize these moments by a common mean over $\mathcal U$. Fitting a Weibull distribution to positive $\sqrt{g_u}$ gives $\widehat k=0.7956516957758913$ and the shared exponent $\beta=\widehat k/4=0.19891292394397284$.

The calibration data used for these statistics comprise 1,048,576 tokens: 220 Magpie, 234 UltraChat and 58 WritingPrompts sequences of length 2,048, rendered with the native chat template. Splitting each source by record-hash parity produces the six equally weighted folds used in the score. Each fold contributes all of its data to local error estimation. Neither the exponent nor candidate weights are refitted on WikiText-2.

\subsection{Storage Budget}

The scores above guide allocation under a whole-model storage constraint that also accounts for tensors with fixed encodings. Per-layer embedding (PLE) tables and token embeddings retain the UD encodings. The output head, routers, QSA indexer query/key, GDN alpha/beta and MHC inject projections retain fixed NINT8 encodings; normalization, convolution and the final gated-residual module retain their native configurations. These fixed tensors consume $B_{\mathrm{fixed}}=33{,}000{,}391{,}107$ bytes. After reserving a further $B_{\mathrm{reserve}}=36{,}436{,}992$ bytes for headers and codebooks, we solve
\begin{equation}
\begin{gathered}
 \min_{c_u\in\mathcal C_u}\;\sum_{u\in\mathcal U}d_u(c_u)\\
 \text{subject to}\;B_{\mathrm{fixed}}+B_{\mathrm{reserve}}\\
 +\sum_{u\in\mathcal U}C_u(c_u)\le B_{\mathrm{UD}}.
\end{gathered}
\label{ns:whole-model-budget}
\end{equation}
We solve this problem using lower-hull allocation and feasible discrete rounding (\secref{sec:allocation}), without imposing per-layer or projection-type ceilings.

\subsection{Evaluation and Results}

We evaluate the resulting whole-model allocations against four released GSQ-RCO checkpoints~\citep{istadaslab2026flashnextgsqrco}. GSQ-RCO combines Gumbel-Softmax scalar quantization~\citep{dadgarnia2026gsq} with Riemannian constrained optimization~\citep{helcig2026rco}. For this comparison, NatScale and GSQ-RCO share 580 frozen WikiText-2 contexts of 512 tokens, the native reference hidden states and output head, and 147,900 predictions at logits 256--510. Full-vocabulary KLD uses FP64 reduction without probability cutoffs; Top1 compares native and candidate argmaxes.

With these evaluation inputs and metrics fixed, each method runs on its corresponding backend. NatScale propagates decoded BF16 candidates through the native Transformers graph with its fixed NINT8 head and recomputed routes. GSQ-RCO retains the released GGUF weights and \texttt{llama.cpp} revision \texttt{3cf03257}, with lossless reference-logit transfer and full-vocabulary metrics.

Alongside these locally evaluated results, \tabref{tab:natscale-whole-main} includes UD scores reported by Unsloth~\citep{unsloth2026flashnext}.

\FloatBarrier
\section*{AI use statement}
Generative AI assisted language editing, literature discovery, \LaTeX{} formatting, and analysis-code review. The authors verified sources, code, and numerical claims, reviewed the manuscript, and take responsibility for all content and artifacts.

\end{document}